\documentclass{article}

\usepackage[final, nonatbib]{neurips_2020}
\usepackage[colorinlistoftodos]{todonotes}
\usepackage{caption}
\usepackage{wrapfig}
\usepackage[utf8]{inputenc} 
\usepackage{comment}
\usepackage[T1]{fontenc}    
\usepackage{hyperref}       
\usepackage{url}            
\usepackage{booktabs}       
\usepackage{amsfonts}       
\usepackage{nicefrac}       
\usepackage{microtype}      
\usepackage{todonotes}
\usepackage{longtable}
\usepackage[ruled,vlined]{algorithm2e}
\usepackage{amsmath}
\usepackage{kky}
\usepackage{subcaption}
\usepackage{afterpage}
\usepackage{float}
\usepackage{ bbold }
\graphicspath{{./figures/}}

\definecolor{darkspringgreen}{rgb}{0.09, 0.45, 0.27}

\title{Modeling Shared Responses in Neuroimaging Studies through MultiView ICA}

\author{%
  Hugo Richard\footnotemark[1] \\
  Inria, Université Paris-Saclay\\
  Saclay, France \\
  \texttt{hugo.richard@inria.fr} \\
  \And
  Luigi Gresele\thanks{Equal contribution} \\
  MPI for Intelligent Systems,\\
  MPI for Biological Cybernetics, Tübingen, Germany \\
  \texttt{luigi.gresele@tuebingen.mpg.de} \\
  \And
  Aapo Hyv{\"a}rinen \\
  Inria, Université Paris-Saclay, Saclay, France\\
  Department of Computer Science HIIT, University of Helsinki, Finland \\
  \texttt{aapo.hyvarinen@helsinki.fi} \\
   \And
  Bertrand Thirion \\
  Inria, Université Paris-Saclay\\
  Saclay, France \\
  \texttt{bertrand.thirion@inria.fr} \\
\And
  Alexandre Gramfort \\
  Inria, Université Paris-Saclay\\
  Saclay, France \\
  \texttt{alexandre.gramfort@inria.fr} \\
 \And
 Pierre Ablin \\
 Département de Mathématiques et Applications\\
 Ecole Normale Supérieure \\
 Paris, France \\
  \texttt{pierre.ablin@ens.fr} \\
}

\begin{document}

\maketitle

\begin{abstract}
  Group studies involving large cohorts of subjects are important to draw general conclusions about brain functional organization.
  However, the aggregation of data coming from multiple subjects is challenging, since it requires accounting for large variability in anatomy, functional topography and stimulus response across individuals.
  Data modeling is especially hard for ecologically relevant  conditions such as  movie watching, where the experimental setup does not imply well-defined cognitive operations.
  We propose a novel MultiView Independent Component Analysis (ICA) model for group studies, where data from each subject are modeled as a linear combination of shared independent sources plus noise.
  Contrary to most group-ICA procedures, the likelihood of the model is available in closed form.
  We develop an alternate quasi-Newton method for maximizing the likelihood, which is robust and converges quickly.
  We demonstrate the usefulness of our approach first on fMRI data, where our model demonstrates improved sensitivity in identifying common sources among subjects.
  Moreover, the sources recovered by our model exhibit lower between-session variability than other methods.
  On magnetoencephalography (MEG) data, our method yields more accurate source localization on phantom data.
  Applied on $200$ subjects from the Cam-CAN dataset it reveals a clear sequence of evoked activity in sensor and source space.
\end{abstract}
\section{Introduction}
\label{sec:intro}
The past decade has seen the emergence of two trends in neuroimaging: the collection of massive neuroimaging datasets, containing data from hundreds of participants~\cite{taylor2017cambridge,van2013wu,sudlow2015uk}, and the use of naturalistic stimuli to move closer to a real life experience with dynamic and multimodal stimuli~\cite{Sonkusare-etal:2019}.
%
%
Large scale datasets provide an unprecedented opportunity to assess the generality and validity of neuroscientific findings across subjects, with the potential of offering novel insights on human brain function and useful medical biomarkers.
%
%
However, when using ecological conditions, such as movie watching or simulated driving, 
stimulations are difficult to quantify. Consequently the statistical analysis of the data using
supervised regression-based approaches is difficult.
This has motivated the use of unsupervised learning methods that leverage the availability of
data from multiple subjects performing the same experiment; analysis on such large groups boosts statistical
power.

Independent component analysis~\cite{hyvarinen2000independent} (ICA) is a widely used unsupervised method for neuroimaging studies. It is routinely applied on individual subject electroencephalography (EEG)~\cite{makeig1996independent}, magnetoencephalography (MEG)~\cite{vigario1998independent} or functional MRI (fMRI)~\cite{mckeown1998independent} data. 
ICA models a set of signals as the product of a \emph{mixing matrix} and a \emph{source} matrix containing independent components.
The identifiability theory of ICA states that having non-Gaussian independent sources is a strong enough condition to recover the model parameters~\cite{comon1994independent}.
ICA therefore does not make assumptions about what triggers brain activations in the stimuli, unlike confirmatory approaches like the general linear model \cite{friston1994statistical, poline2012general}.
This explains why, in fMRI processing, it is a model of choice when analysing
resting state data \cite{beckmann2005investigations} or when subjects are
exposed to natural~\cite{malinen2007towards}~\cite{bartels2005brain} or complex stimuli such as simulated driving \cite{calhoun2002different}.
In M/EEG processing, it is widely used to isolate acquisitions artifacts from neural signal~\cite{jung1998extended}, and to identify brain sources of interest~\cite{vigario2000independent, delorme2012independent}.

%
%
However, unlike with univariate methods, statistical inference about multiple subjects using ICA is not straightforward: so-called group-ICA is the topic of various studies~\cite{hyvarinen2013independent}.
Several works assume that the subjects share a common mixing matrix, but with different sources~\cite{pfister2019robustifying}~\cite{svensen2002ica}.
Instead, we focus on a model where the subjects share a common sources matrix, but have different mixing matrices.
When the subjects are exposed to the same stimuli, the common source matrix corresponds to the group \emph{shared responses}.
%
Most methods proposed in this framework proceed in two steps~\cite{calhoun2009review, huster2015group}.
First, the data of individual subjects are aggregated into a single dataset, often resorting to dimension reduction techniques like Principal Component Analysis (PCA).
Then, off-the-shelf ICA is applied on the aggregated dataset.
This popular method has the advantage of being simple and
straightforward to implement since it resorts to customary single-subject
ICA method.
However, it is not grounded in a principled probabilistic model of the problem, and does not have strong statistical guarantees like asymptotic efficiency. 
%
%
%

We propose a novel group ICA method called \emph{MultiView ICA}.
It models each subject's dataset as a linear combination of a common
sources matrix with additive Gaussian noise.
Importantly, we consider that the noise is on the sources and not on
the sensors.
This greatly simplifies the likelihood of the model which can even be
written in closed-form.
%
Despite its simplicity, our model allows for an expressive representation of inter-subject variability through subject-specific functional topographies (mixing matrices) and variability in the individual response (with noise in the source domain).
%
%
To the best of our knowledge, this is the first time that such a tractable likelihood is proposed for multi-subject ICA.
The likelihood formulation shares similarities with the usual ICA likelihood, which allows us to develop a fast and robust alternate quasi-Newton method for its maximization.

\textbf{Contribution}
In section~\ref{sec:mvica}, we introduce the MultiView ICA model, and show that it is identifiable. We then write its likelihood in closed form, and maximize it using an alternate quasi-Newton method.
We also provide a sensitivity analysis for MultiView ICA, and show that the choice of the noise parameter in the algorithm has little influence on the output.
In section~\ref{sec:rel_work}, we compare our approach to other group ICA methods.
Finally, in section~\ref{sec:expts}, we empirically verify through extensive experiments on fMRI and MEG data that it improves source identification with respect to competing methods, suggesting that the expressiveness and robustness of our model make it a useful tool for multivariate neural signal analysis.

\section{Multiview ICA for Shared response modelling}
\label{sec:mvica}
\textbf{Notation} The absolute value of the determinant of a matrix $W$ is $|W|$.
The $\ell_2$ norm of a vector $\sbb$ is $\|\sbb\|$.
For a scalar valued function $f$ and a vector $\sbb \in \bbR^k$, we write $f(\sbb) = \sum_{j=1}^kf(s_j)$ and denote $f'$ the gradient of $f$.
\emph{All proofs are deferred to appendix~\ref{sec:app_proofs}}.
\subsection{Model, likelihood and approximation}
Given $m$ subjects, we model the data $\xb^i\in\bbR^k$ of subject $i$ as
\begin{equation}
\label{eq:ica_model}
\boxed{
    \xb^i = A^i(\sbb + \nb^i), \enspace i=1,\dots, m
    }
\end{equation}
where $\sbb = [s_1, \dots, s_k]^{\top} \in \bbR^k$ are the shared independent sources, $\nb^i \in \bbR^k$ is individual noise, $A^i \in \bbR^{k\times k}$ are the individual mixing matrices, assumed to be full-rank.
We assume that samples are observed i.i.d. For simplicity, we assume that the sources share the same density $d$, so that the independence assumption is $p(\sbb) = \prod_{j=1}^k d(s_j)$. Finally, we assume that the noise is Gaussian decorrelated of variance $\sigma^2$, $\nb^i \sim \mathcal{N}(0, \sigma^2I_k)$, and that the noise is independent across subjects and independent from the sources.
The assumption of additive white noise on the sources models individual deviations from the shared sources $\sbb$.
It is equivalent to having noise on the sensors with covariance $\sigma^2 A^i \left(A^i\right)^{\top}$, i.e. a scaled version of the data covariance without noise.

Since the sources are shared by the subjects, there are many more observed variables than sources in the model: there are $k$ sources, while there are $k \times m$ observations.
Therefore, model~\eqref{eq:ica_model} can be seen as an instance of \emph{undercomplete} ICA.
The goal of multiview ICA is to recover the mixing matrices $A^i$ from observations of the $\xb^i$.
The following proposition extends the standard idenfitiability theory of ICA~\cite{comon1994independent} to multiview ICA, and shows that recovering the sources/mixing matrices is a well-posed problem up to scale and permutation.
\begin{proposition}[Identifiability of MultiView ICA]
\label{prop:identifiability}
Consider $\xb^i, \enspace i=1\dots m,$ generated from~\eqref{eq:ica_model}. Assume that $\xb^i = A'^i(\sbb' + \nb'^i)$ for some invertible matrices $A'^i\in \bbR^{k\times k}$, independent non-Gaussian sources $\sbb'\in \bbR^k$ and Gaussian noise $\nb'^i$. Then, there exists a scale and permutation matrix $P\in \bbR^{k\times k}$ such that for all $i$, $A'^i = A^i P$.
\end{proposition}
%
%
%
%
%
%
%
We propose a maximum-likelihood approach to estimate the mixing matrices. 
We denote by $W^i = (A^i)^{-1}$ the unmixing matrices, and view the likelihood as a function of $W^i$ rather than $A^i$. As shown in Appendix~\ref{sec:appendix:likelihood_transform}, the negative log-likelihood can be written by integrating over the sources
\begin{equation} 
    \label{eq:likelihood}
    \loss(W^1, \dots, W^m) = -\sum_{i=1}^m\log|W^i| - \log\left(\int_{\sbb}\exp\left(-\frac1{2\sigma^2}\sum_{i=1}^m\|W^i\xb^i - \sbb\|^2\right)p(\sbb)d\sbb\right),
\end{equation}
up to additive constants.
%
Since this integral factorizes, i.e.\ the integrand is a product of functions of $s_j$, we can perform the integration as shown in Appendix~\ref{sec:appendix:integration}. We define a smoothened version of the logarithm of the source density $d$ by convolution with a Gaussian kernel as
$
    f(s)= \log \left(\int \exp(-\frac{m}{2\sigma^2} z^2) d(s-z) dz\right)
$
and $\tilde{\sbb} = \frac1m\sum_{i=1}^m W^i\xb^i$ the source estimate.
The negative log-likelihood becomes
\begin{equation}
    \label{eq:cost_function}
    \loss(W^1, \dots, W^m) = -\sum_{i=1}^m \log|W^i| + \frac1{2\sigma^2}\sum_{i=1}^m\|W^i\xb^i - \tilde{\sbb}\|^2 + f(\tilde{\sbb}).
\end{equation}
Multiview ICA is then performed by minimizing $\loss$, and the estimated shared sources are $\tilde{\sbb}$.
The negative log-likelihood $\loss$ is quite simple, and importantly, can be computed easily given the parameters of the model and the data; it does not involve any intractable integral.

For one subject ($m=1$), $\loss(W^1)$ simplifies to the negative log-likelihood of ICA and we recover Infomax~\cite{bell1995information,cardoso1997infomax}, where the source log-pdf is replaced with the smoothened $f$.
%
\subsection{Alternate quasi-Newton method for MultiView ICA}
The parameters of the model are estimated by minimizing $\loss$.
We propose a combination of quasi-Newton method and alternate minimization for this task.
First, $\mathcal{L}$ is non-convex: it is only defined when the $W^i$ are invertible, which is a non-convex set.
Therefore, we only look for local minima as usual in ICA.
We propose an alternate minimization scheme, where $\loss$ is alternatively diminished with respect to each $W^i$. 
When all matrices $W^1, \dots, W^m$ are fixed but one, $W^i$, $\loss$ can be rewritten, up to an additive constant 
\begin{equation}
    \label{eq:indiv_loss}
    \loss^i(W^i) = -\log|W^i| + \frac{1 - 1/m}{2\sigma^2}\|W^i\xb^i - \frac{m}{m-1}\tilde{\sbb}^{-i}\|^2 + f(\frac1m W^i \xb^i +\tilde{\sbb}^{-i}), 
\end{equation}
with $\tilde{\sbb}^{-i} = \frac1m \sum_{j \neq i}W^j \xb^j$.
This function has the same structure as the usual maximum-likelihood ICA cost function: it is written $\loss^i(W^i) = -\log|W^i| + g(W^i\xb^i)$, where $g(\yb) = \sum_{j=1}^kf(\frac{y_j}{m} + \tilde{\sbb}^{-i}_j) + \frac{1 - 1/m}{2\sigma^2}(y_j - \frac{m}{m-1}\tilde{\sbb}^{-i}_j)^2$.
Fast quasi-Newton algorithms ~\cite{zibulevsky2003blind, ablin2018faster} have been proposed for minimizing such functions.
We employ a similar technique as~\cite{zibulevsky2003blind}, which we now describe.

Quasi-Newton methods are based on approximations of the Hessian of $\loss^i$.
The relative gradient (resp. Hessian)~\cite{amari1996new, cardoso1996equivariant} of $\loss^i$ is defined as the matrix $G^i\in \bbR^{k \times k}$ (resp. tensor $\mathcal{H}^i \in \bbR^{k\times k\times k\times k}$) such that as the matrix $E\in\bbR^{k\times k}$ goes to $0$, we have $\loss^i((I_k + E)W^i) \simeq \loss^i(W^i) + \langle G^i, W^i\rangle + \frac12\langle E, \mathcal{H}^iE\rangle$. Standard manipulations yield:
\begin{equation}
    \label{eq:gradient}
    G^i = \frac1mf'(\tilde{\sbb})(\yb^i)^{\top} + \frac{1 - 1 /m}{\sigma^2}(\yb^i - \frac{m}{m-1}\tilde{\sbb}^{-i})(\yb^i)^{\top} - I_k, \text{ where } \yb^i= W^i\xb^i
\end{equation}
\begin{equation}
    \label{eq:hessian}
    \mathcal{H}^i_{abcd} = \delta_{ad}\delta_{bc} + \delta_{ac}\left(\frac{1}{m^2}f''(\tilde{\sbb}_a) + \frac{1 - 1/m}{\sigma^2}\right)\yb^i_{b}\yb^i_d,\enspace \text{for }a, b, c, d =1\dots k
\end{equation}

Newton's direction is then $-\left(\mathcal{H}^i\right)^{-1}G^i$. However, this Hessian is costly to compute (it has $\simeq k^3$ non-zero coefficients) and invert (it can be seen as a big $k ^2\times k^2$ matrix). Furthermore, to enforce that Newton's direction is a descent direction, the Hessian matrix should be regularized in order to eliminate its negative eigenvalues~\cite{nocedal2006numerical}, and $\mathcal{H}^i$ is not guaranteed to be positive definite.
These obstacles render the computation of Newton's direction impractical.
Luckily, if we assume that the signals in $\yb^i$ are independent, severall coefficients cancel, and the Hessian simplifies to the approximation
\begin{equation}
    \label{eq:hessian_approx}
    H^i_{abcd} = \delta_{ad}\delta_{bc} + \delta_{ac}\delta_{bd}\Gamma^i_{ab}\enspace \text{with  }\Gamma^i_{ab} = \left(\frac{1}{m^2}f''(\tilde{\sbb}_a) + \frac{1 - 1/m}{\sigma^2}\right)\left(\yb^i_{b}\right)^2.
\end{equation}
This approximation is sparse: it only has $k(2k -1)$ non-zero coefficients.
In order to better understand the structure of the approximation, we can compute the matrix $\left(H^iM\right)$ for $M\in \bbR^{k\times k}$. 
We find $\left(H^iM\right)_{ab} = \Gamma^i_{ab}M_{ab} + M_{ba}$: $H^iM_{ab}$ only depends on $M_{ab}$ and $M_{ba}$, indicating a simple block diagonal structure of $H^i$.
The tensor $H^i$ is therefore easily regularized and inverted:
$\left((H^i)^{-1}M\right)_{ab} = \frac{\Gamma^i_{ba}M_{ab} - M_{ba}}{\Gamma^i_{ab}\Gamma^i_{ba} - 1}$.
Finally, since this approximation is obtained by assuming that the $\yb^i$ are independent, the direction $-\left(H^i\right)^{-1}G^i$ is close to Newton's direction when the $\yb^i$ are close to independence, leading to fast convergence.
Algorithm~\ref{algo:mv_ica} alternates one step of the quasi-Newton method for each subject until convergence.
A backtracking line-search is used to ensure that each iteration leads to a decrease of $\loss^i$.
The algorithm is stopped when maximum norm of the gradients over one pass on each subject is below some tolerance level, indicating that the algorithm is close to a stationary point.

\begin{algorithm}[H]
\label{algo:mv_ica}
\SetAlgoLined
\KwIn{Dataset $(\xb^i)_{i=1}^m$, initial unmixing matrices $W^i$, noise parameter $\sigma$, function $f$,  tolerance $\varepsilon$}
Set tol$=+\infty$, $\tilde{\sbb} = \frac1m\sum_{i=1}^kW^i\xb^i$\\
 \While{\text{tol}$>\varepsilon$}{
 tol = 0 \\
  \For{$i=1\dots m$}{
  Compute $\yb^i = W^i \xb^i$, $\tilde{\sbb}^{-i} = \tilde{\sbb} - \frac1m\yb^i$, gradient $G^i$ (eq.~\eqref{eq:gradient}) and Hessian $H^i$ (eq.~\eqref{eq:hessian_approx})\\
  Compute the search direction $D = -\left(H^i\right)^{-1}G^i$\\
  Find a step size $\rho$ such that $\loss^i((I_k + \rho D)W^i) < \loss^i(W^i)$ with line search\\
  Update $\tilde{\sbb} = \tilde{\sbb} + \frac{\rho}{m} DW^i \xb^i$, $W^i = (I_k + \rho D)W^i$, tol$=\max($tol$,\|G^i\|)$\\
  }
 }
 \Return{Estimated unmixing matrices $W^i$, estimated shared sources $\tilde{\sbb}$}
 \caption{Alternate quasi-Newton method for MultiView ICA}
\end{algorithm}
\subsection{Robustness to model misspecification}
Algorithm~\ref{algo:mv_ica} has two hyperparameters: $\sigma$ and the function $f$.
The latter is usual for an ICA algorithm, but the former is not.
We study the impact of these parameters on the separation capacity of the algorithm, when these parameters do not correspond to those of the generative model~\eqref{eq:ica_model}.
\begin{proposition}
\label{prop:robust}
We consider the cost function $\loss$ in eq.~\eqref{eq:cost_function} with noise parameters $\sigma$ and function $f$.
Assume sub-linear growth on $f'$: $|f'(x)|\leq c|x|^{\alpha} + d$ for some $c, d > 0$ and $0<\alpha<1$.
Assume that $\xb^i$ is generated following model~\eqref{eq:ica_model}, with noise parameter $\sigma'$ and density of the source $d'$ which need not be related to $\sigma$ and $f$.
Then, there exists a diagonal matrix $\Lambda$ such that $(\Lambda (A^1)^{-1}, \dots, \Lambda (A^m)^{-1})$ is a stationary point of $\loss$, that is $G^1,\dots, G^m =0$ at this point.
\end{proposition}
The sub-linear growth of $f'$ is a customary hypothesis in ICA which implies that $d$ has heavier-tails than a Gaussian, and in appendix~\ref{ref:robust} we provide other conditions for the result to hold.
In this setting, the shared sources estimated by the algorithm are $\tilde{\sbb} = \Lambda (\sbb + \frac1m \sum_{i=1}^m \nb^i)$, which is a scaled version of the best estimate of the shared sources under the Gaussian noise hypothesis.

This proposition shows that, up to scale, the true unmixing matrices are a stationary point for Algorithm~\ref{algo:mv_ica}: if the algorithm starts at this point it will not move.
The question of stability is also interesting: if the algorithm is initialized ~\emph{close} to the true unmixing matrices, will it converge to the true unmixing matrix?
In the appendix~\ref{sec:stability}, we provide an analysis similar to~\cite{cardoso1998blind}, and derive sufficient numerical conditions for the unmixing matrices to be local minima of $\mathcal{L}$.
We also study the practical impact of changing the hyperparameter $\sigma$ on the accuracy of a machine learning pipeline based on MultiviewICA on real fMRI data in the appendix Sec.~\ref{sec:app_sigma_impact}.
As expected from the theoretical study, the performance of the algorithm is barely affected by $\sigma$.
\subsection{Dimensionality reduction}
So far, we have assumed that the dimensionality of each view (subject) and that of the sources is the same. This reflects the standard practice in ICA of having equal number of observations and sources. 
In practice, however, we might want to estimate fewer sources than there are observations per view; the original dimensionality of the data 
might in practice not be computationally tractable.
The problem of how to perform subject-wise dimensionality reduction in group studies 
is an interesting one \emph{per se}, and out of the main scope of this work. For our purposes, it can be considered as a preprocessing step for which well-known various solutions can be applied. 
We discuss this further in section~\ref{sec:rel_work} and in appendix~\ref{sec:app_rel_work}.
\section{Related Work}
\label{sec:rel_work}
Many methods for data-driven multivariate analysis of neuroimaging group studies have been proposed. We summarize the characteristics of some of the most commonly used ones. A more thorough description of these methods can be found in appendix~\ref{sec:app_rel_work}.
For completeness, we start by describing PCA. For a zero-mean data matrix $X$ of size $p\times n$ with $p \leq n$, we denote $X= UD V^{\top}$ the singular value decomposition of $X$ where $U \in \bbR^{p\times p}$, $V \in \bbR^{n \times p}$ are orthogonal and $D$ the diagonal matrix of singular values ordered in decreasing order.
The PCA of $X$ with $k$ components is $Y\in\bbR^{k\times n}$ containing the first $k$ rows of $DV^{\top}$, and it does not hold in general that $YY^{\top}=I_k$: for the rest of the paper, what we call PCA does not include whitening of the signals.

\textbf{Group ICA} When datasets are high-dimensional, a three steps procedure is often used: first dimensionality reduction is performed on data of each subject  separately; then the reduced data are merged into a common representation; finally, an ICA algorithm is applied for shared source extraction. The merging of the reduced data is often done by PCA \cite{calhoun2001method} or multi set CCA \cite{varoquaux2009canica}.
This is a popular method for fMRI~\cite{calhoun2009review} and EEG~\cite{eichele2011eegift} group studies.
These methods directly recover only group level, shared sources; when individual sources are needed, additional steps are required (back-projection \cite{calhoun2001method} or dual-regression \cite{beckmann2009group}).
In contrast, MultiView ICA finds individual and shared independent components in a single step.
Finally, in contrast to the methods described above, our method maximizes a likelihood, which brings statistical guarantees like consistency or asymptotic efficiency.
The SR-ICA approach of \cite{zhang2016searchlight} performs dimension reduction, merging and independent component estimation. It is therefore similar to our method.
However, they propose to modify the FastICA algorithm~\cite{hyvarinen1999fast} in a rather heuristic way, without specifying an optimization problem, let alone maximizing a likelihood. In the experiments on fMRI data in appendix~\ref{appendix_reproduce}, we obtain better performance with MultiView ICA than the reported performance of SR-ICA.

\textbf{Likelihood-based models} One can consider the more general model $\xb^i = A^i\sbb^i + \nb^i$, where the noise covariance can be learned from the data~\cite{guo2008unified}.
The likelihood for this model involves an intractable high dimensional integral that is cumbersome to evaluate, and is then optimized with the Expectation-Maximization (EM) algorithm, which is known to converge slowly and unreliably~\cite{bermond1999approximate, petersen2005slow}.
Having the simpler model~\eqref{eq:ica_model} leads to a closed-form likelihood, that can then be optimized by more efficient means than the EM algorithm.
In model~\eqref{eq:ica_model}, the noise can be interpreted as individual variability rather than sensor noise. 
In appendix~\ref{app:complex_cov}, we generate data following model $\xb^i = A^i\sbb^i + \nb^i$ and report the reconstruction error. The difference in performance between algorithms is small. 

\textbf{Structured mixing matrices} One strength of our model is that we only assume that the mixing matrices are invertible and still enjoy identifiability whereas some other approaches impose additional constraints. For instance tensorial methods~\cite{beckmann2005tensorial} assume that the mixing matrices are the same up to diagonal scaling.
Other methods impose a common mixing matrix~\cite{cong2013validating, grin2010independent, calhoun2001fmri, Monti18UAI}. Like PCA, the Shared Response Model~\cite{chen2015reduced} (SRM) assumes orthogonality of the mixing matrices. While the model defines a simple likelihood and provides an efficient way to reduce dimension, the SRM model is not identifiable as shown in appendix~\ref{sec:app_identifiability}, and the orthogonal constraint may not be plausible.

\textbf{Matching sources a posteriori} A different path to multi-subject ICA is to extract independent components with individual ICA in each subject and align them. We propose a simple baseline approach to do so called \emph{PermICA}.
Inspired by the heuristic of the hyperalignment method~\cite{haxby2011common} we choose a reference subject and first match the sources of all other subjects to the sources of the reference subject. The process is then repeated multiple times, using the average of previously aligned sources as a reference. Finally, group sources are given by the average of all aligned sources. We use the Hungarian algorithm to align pairs of mixing matrices~\cite{tichavsky2004optimal}.
Alternative approaches involving clustering have also been developed~\cite{esposito2005independent,bigdely2013measure}.

\textbf{Deep Learning} Deep Learning methods, such as convolutional auto-encoders (CAE), can also be used to find the subject specific unmixing~\cite{chen2016convolutional}. While these nonlinear extensions of the aforementioned methods are interesting, these models are hard to train and interpret. In the experiments on fMRI data in appendix~\ref{appendix_reproduce}, we obtain better accuracy with MultiView ICA than that of CAE reported in~\cite{chen2016convolutional}.

\textbf{Correlated component analysis} Other methods can be used to recover the shared neural responses such as the correlated component approach of Dmochowski~\cite{dmochowski2012correlated}. We benchmark our method against its probabilistic version~\cite{kamronn2015multiview} called BCorrCA in Figure~\ref{fig:meg}. Our method yields much better results. 

\textbf{Autocorrelation} Another way to perform ICA is to leverage spectral diversity of the sources rather than non-Gaussianity.
These methods are popular alternative to non-Gaussian ICA in the single-subject setting~\cite{tong1991indeterminacy, belouchrani1997blind, pham1997blind} and they output significantly different sources than non-Gaussian ICA~\cite{delorme2012independent}.
Extensions to multiview problems have been proposed~\cite{lukic2002ica, congedo2010group}.
\vspace{-5pt}
\section{Experiments}
\label{sec:expts}
All code for the experiments is written in Python.
We use Matplotlib for plotting~\cite{hunter2007matplotlib} , scikit-learn for
machine-learning pipelines~\cite{pedregosa2011scikit}, MNE for MEG
processing~\cite{gramfort2013meg}, Nilearn for fMRI processing and for its CanICA implementation~\cite{abraham2014machine}, Brainiak~\cite{kumar2020brainiak} for its SRM implementation. 
In the following, the noise parameter in MultiviewICA is always fixed to $\sigma =1$.
We use the function $f(\cdot)= \log\cosh(\cdot)$, giving the non-linearity $f'(\cdot) = \tanh(\cdot)$.
We use the Infomax cost function~\cite{bell1995information} with the same non-linearity to perform standard ICA, with the Picard algorithm~\cite{ablin2018faster} for fast and robust minimization of the cost function. Picard is applied with the default hyper-parameters.
The code for MultiViewICA is available online at \url{https://github.com/hugorichard/multiviewica}.

We compare the following methods to obtain $k$ components:
\emph{GroupPCA} is PCA on spatially concatenated data. It corresponds to a transposed version of \cite{smith2014group}.
\emph{PermICA} is described in the previous section.
\emph{SRM} is  the algorithm of~\cite{chen2015reduced}.
\emph{GroupICA} is ICA applied after GroupPCA.
\emph{PCA+GroupICA} corresponds to GroupICA applied on subject data that have been first individually reduced by PCA with $k$ components. 
These two approaches correspond to transposed versions of~\cite{calhoun2001fmri}, and are similar to \cite{eichele2011eegift}.
\emph{CanICA} corresponds to PCA+GroupICA where the merging is done using multi set CCA rather than PCA. The dimension reduction in MultiView ICA and PermICA is performed with SRM in fMRI experiments and subject-specific PCA in MEG experiments. Initialization is discussed in appendix~\ref{sec:app_init}. A summary of our quantitative results on real data is available in appendix~\ref{sec:app_real_data}.

\begin{wrapfigure}{l}{0.4\textwidth}
\label{fig:synth}
\includegraphics[width=0.38\textwidth]{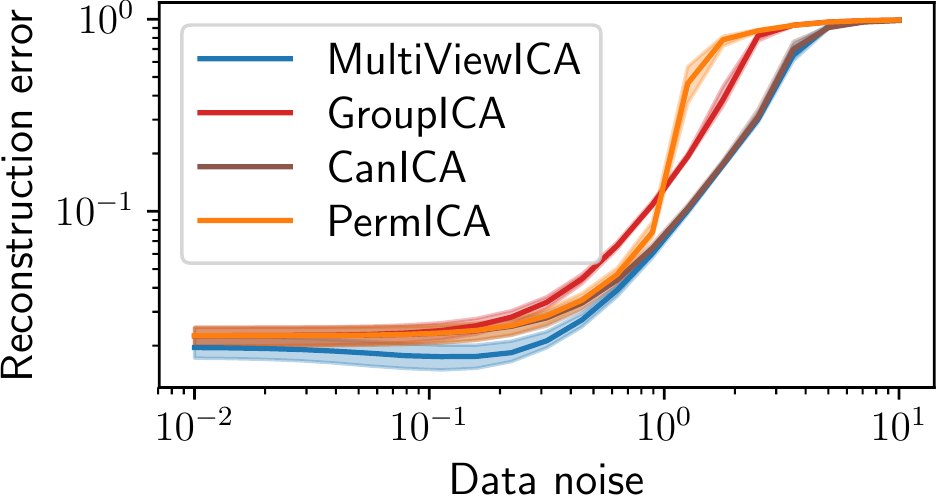}
\captionof{figure}{\textbf{Synthetic experiment}: reconstruction error of the algorithms on data following model~\eqref{eq:ica_model}.}
\end{wrapfigure}

\textbf{Synthetic experiment}
We validate our method on synthetic data generated according to the model in equation~\eqref{eq:ica_model}.
The sources are generated i.i.d. from a Laplace density $d(x)=\frac12\exp(-|x|)$.
The mixing matrices $A^1,\cdots, A^m$ are generated with i.i.d. entries following a normal law.
Each compared algorithm returns a sequence of estimated unmixing matrices $W^1, \dots, W^m$.
The performance of an algorithm is measured by the reconstruction error between the estimated sources and the true sources.
%
%
We use $m=10$ datasets, $k=15$ sources and $n=1000$ samples. Each experiment is repeated with $100$ random seeds.
We vary the noise level in the data generation from $10^{-2}$ to $10$.

Multiview ICA has uniformly better performance than the other algorithms, which illustrates the strength of maximum-likelihood based methods. In accordance with results of section~\ref{sec:mvica}, it is able to separate the sources even with misspecified noise parameter and source density.
%

\textbf{fMRI data and preprocessing} 
We evaluate the performance of our approach on four different fMRI datasets.
The \emph{sherlock} dataset~\cite{chen2017shared} contains recordings of 16 subjects watching an episode of the BBC TV show "Sherlock" (50 mins).
The \emph{forrest} dataset~\cite{hanke2014high} was collected while 19 subjects were listening to an auditory version of the film "Forrest Gump" (110 mins).
The \emph{clips} dataset~\cite{ibc} was collected while 12 participants were exposed to short video clips (130 mins).
The \emph{raiders} dataset~\cite{ibc} was collected while 11 participants were watching the movie "Raiders of the Lost Ark" (110 mins).
The \emph{raiders-full} dataset~\cite{ibc} is an extension of the \emph{raiders} dataset where the first two scenes of the movie are shown twice (130 mins).
Like \cite{zhang2016searchlight}, we used full brain data. The rest of the preprocessing is identical to \cite{chen2017shared}. See \ref{preprocessing} for a detailed description of the datasets and preprocessing steps. Unless stated otherwise we use spatially unsmoothed data, except for the \emph{sherlock} dataset, for which the available data are already preprocessed with a 6\,mm spatial smoothing. All datasets are built from successive acquisitions called \emph{runs} that typically last 10 minutes each.
We define the chance level as the performance of an algorithm that computes unmixing matrices and projections to lower dimensional space by sampling random numbers from a standard normal distribution. 

\textbf{Reconstructing the BOLD signal of missing subjects}
We want to show that once unmixing matrices have been learned, they can be
used to predict
evoked responses across subjects, which can be useful to perform transfer learning~\cite{zhang2018transfer}.
We split the data into three groups. First, we randomly choose $80\%$ of all runs from all subjects to form the training set.
Then, we randomly choose $80\%$ of subjects and take the remaining $20\%$  runs as testing set.
The left-out runs  of the remaining $20\%$ subjects form the validation set.
The compared algorithms are run on the training set and evaluated using the testing and validation sets.
After an algorithm is run on training data, it defines for each subject a \emph{forward operator} that maps individual data to the source space and a \emph{backward operator} that maps the source space to individual data. For instance in ICA the forward operator is the product of the dimensionality reduction projection and unmixing matrix.
We estimate the shared responses on the testing set by applying the forward operators on the testing data and averaging. Finally, we reconstruct the individual data from subjects in the validation set by applying the backward operators to the shared responses. We measure the difference between the true signal and the reconstructed one using voxel-wise $R^2$ score. The $R^2$ score between two series $\xb \in \bbR^n$ and $\yb \in \bbR^n$ is defined as
$R^2(\xb, \yb) = 1 - \frac1{n\Var(\yb)}\sum_{t=1}^n (x_t - y_t)^2$, where $\Var(\yb) = \frac1n\sum_{t=1}^n (y_t - \frac1n \sum_{t'=1}^n y_{t'})^2$ is the empirical variance of $\yb$.
The $R^2$ score is always smaller than $1$, and equals $1$ when $\xb= \yb$.
The experiment is repeated 25 times with random splits to obtain error bars.

In this experiment we apply a 6\,mm spatial smoothing to all datasets. The $R^2$ score per
voxel depends heavily on which voxel is considered. For example voxels in the
visual cortex are better reconstructed in the \emph{sherlock} dataset than in
the \emph{forrest} dataset (see Figure~\ref{fig:brainmaps} in appendix~\ref{brainmaps}). In Figure~\ref{fig:reconstruction}
(top) we plot the mean $R^2$ score inside a region of interest (ROI) in order to leave out regions where there is no useful information.
ROIs are chosen based on the performance of GroupICA (more
details in appendix~\ref{brainmaps}).
MultiView ICA has similar or better performance than the other methods on all datasets.
This demonstrates its ability to capture inter-subject variability, making it a candidate of choice to handle missing data or perform transfer learning.

\begin{figure}
  \centering
  \includegraphics[width=0.9\textwidth]{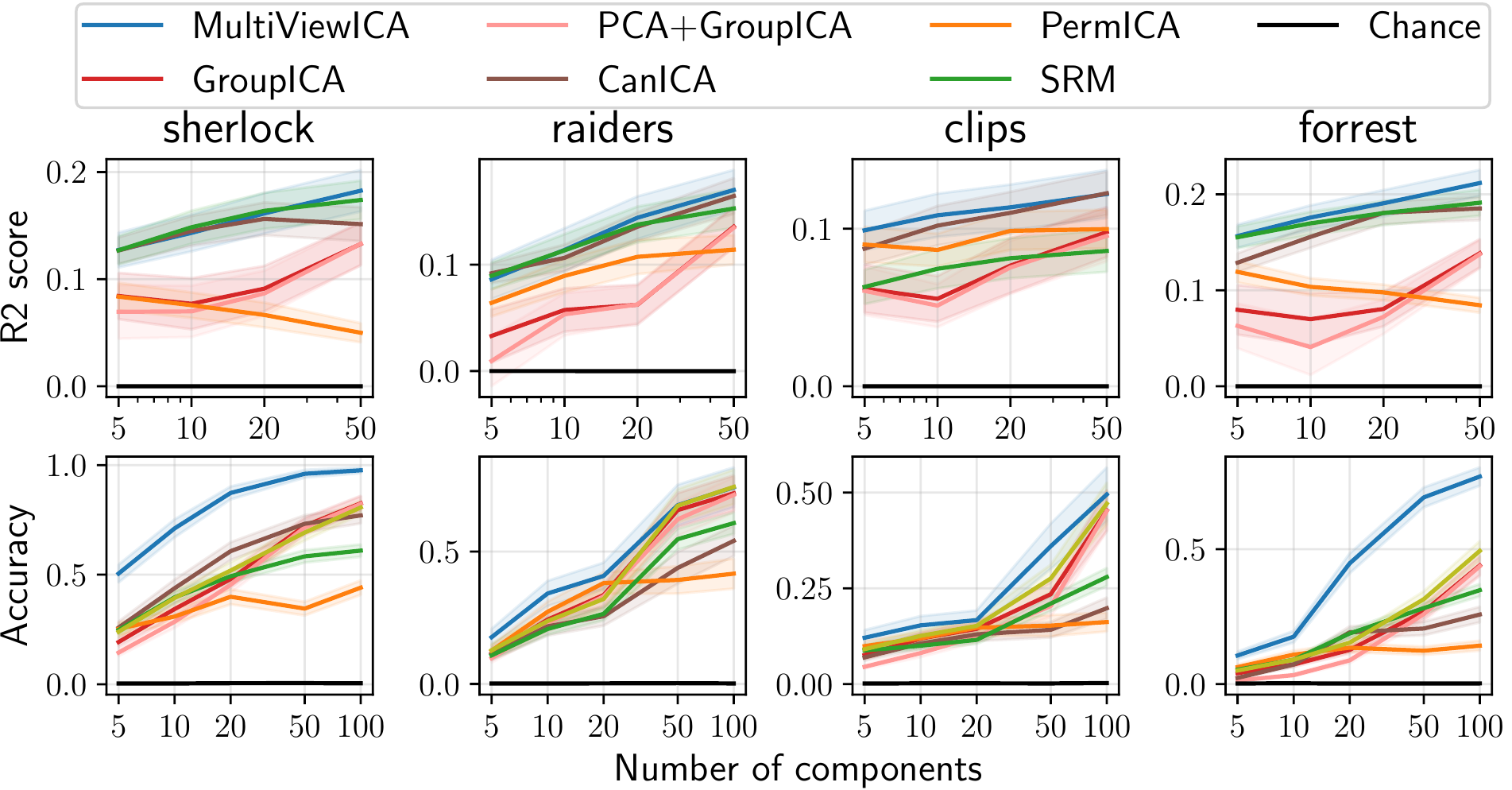}
  \caption{\emph{Top:} \textbf{Reconstructing the BOLD signal of
      missing subjects}. Mean $R^2$ score between reconstructed data and true
    data (higher is better). \emph{Bottom:} \textbf{Between subjects time-segment matching}. Mean
    classification accuracy. Error bars represent a 95 \% confidence interval over cross validation splits.}
  \label{fig:reconstruction}
  \label{fig:timesegment}
\end{figure}

\textbf{Between subjects time-segment matching} 
We reproduce the time-segment matching experiment of
\cite{chen2015reduced}. 
We split the runs into a train and test set. After fitting the model on the training set, we apply the forward operator of each subject on the test set yielding individual sources matrices. We estimate the shared responses by averaging the individual sources of each subjects but one.  We select a target time-segment (9 consecutive timeframes) in the shared responses and try to localize the corresponding time segment in the sources of the left-out subject using a maximum-correlation classifier.
This is a standard evaluation of SRM-like methods also used in  \cite{chen2015reduced}, \cite{guntupalli2018computational}, \cite{Nastase741975} or
\cite{zhang2016searchlight}.
The time-segment is said to be
correctly classified if the correlation between the sample and target
time-segment is higher than with any other time-segment (partially overlapping time windows are excluded).
We use 5-Fold cross-validation across runs: the training set contains 80\% of the runs and the test set 20\%, and repeat the experiment using all possible choices for left-out subjects. 
The mean accuracy is reported in Figure~\ref{fig:timesegment} (bottom). 
MultiView ICA yields a consistent and substantial improvement in accuracy compared to other methods on the four datasets. We see a marked improvement on the datasets sherlock and forrest. A possible explanation lies in the preprocessing pipeline. Sherlock data undergo a 6~mm spatial smoothing and Forrest data are acquired at a higher resolution (7T vs 3T for other data). This affects the signal to noise ratio.
In appendix~\ref{sec:app_sigma_impact}, we compute the accuracy of MultiviewICA on the sherlock dataset with 10 components when the noise parameter varies. MultiviewICA performs consistently well for a wide range of noise parameter values, and only breaks at very high values. It supports the theoretical claim of Prop~\ref{prop:robust} that the noise parameter is of little importance.

In appendix~\ref{app_spatialmaps}, we present a variation of this experiment.  We measure the ability of each algorithm to extract meaningful shared sources that correlate more when they correspond to the same stimulus than when they correspond to distinct stimuli and show the improved performance of MultiView ICA. 
In appendix~\ref{sec:spatial_maps}, we plot the average forward operator across subjects of MultiView ICA and GroupICA with 5 components on the forrest, sherlock, raiders and clips datasets.

\textbf{Phantom MEG data}
We demonstrate the usefulness of our approach on MEG data.
The first experiment uses data collected with a realistic head phantom, which is a plastic device mimicking real electrical brain sources.
Eight current dipoles positioned at different locations can be switched on or off.
We view each dipole as a subject and therefore have $m=8$.
We only consider the 102 magnetometers.
An epoch corresponds to 3\,s of MEG signals where a dipole is switched on for 0.4\,s with an oscillation at 20\,Hz and a peak-to-peak amplitude of 200\,nAm.
This yields a matrix of size $p\times n$ where $p=102$ is the number of sensors, and $n$ is the number of time samples.
We have access to $100$ epochs per dipole.
For each dipole, we chose $N_e=2, \dots, 16$ epochs at random among our set of 100 epochs and concatenate them in the temporal dimension.
We then apply algorithms on these data to extract $k=20$ shared sources.
As we know the true source (the timecourse of the dipole), we can compute the reconstruction error of each source as the squared norm of the difference between the estimated source and the true source, after normalization to unit variance and fixing the sign.
We only retain the source of minimal error.
We also estimate for each forward operator the localization of the source by performing dipole fitting using its column corresponding to the source of minimal error.
We then compute the distance of the estimated dipole to the true dipole.
These metrics are reported in figure~\ref{fig:meg} when the number of epochs considered $N_e$ varies.
We also compare our method to the Bayesian Canonical Correlation Analysis (BCorrCA) of~\cite{kamronn2015multiview}.
On this task, BCorrCA is outperformed by ICA methods.
MultiView ICA requires fewer epochs to correctly reconstruct and localize the true source.

\textbf{Experiment on Cam-CAN dataset}
Finally, we apply MultiView ICA on the Cam-CAN dataset~\cite{taylor2017cambridge}. We use the magnetometer data from the MEG of $200$ subjects.
Each subject is repeatedly presented an audio-visual stimulus. 
The MEG signal corresponding to these trials are then time-averaged to isolate the evoked response, yielding individual data.
The MultiView ICA algorithm is then applied to extract $20$ shared sources.
$9$ sources were found to correspond to noise by visual inspection, and the $11$ remaining are displayed in figure~\ref{fig:meg}.
We observe that MultiView ICA recovers a very clean sequence of evoked potentials with sharp peaks
for early components and slower responses for late components.
In order to visualize their localization, we perform source localization for each subject by solving the inverse problem using sLORETA~\cite{pascual2002standardized}, providing a source estimate for each source.
Then, we register each source estimate to a common reference brain.
Finally, the source estimates are averaged, and thresholded maps are displayed in figure~\ref{fig:meg}.
Individual maps corresponding to each source are displayed in appendix~\ref{sec:app_montages}.
The figure highlights both early auditory and visual cortices, also suggesting a propagation
of the activity towards the ventral regions and higher level visual areas.

\begin{figure}
    \centering
    \begin{minipage}{\linewidth}
      \begin{minipage}{0.55\linewidth}
        \begin{subfigure}{\textwidth}
            \includegraphics[width=0.99\linewidth]{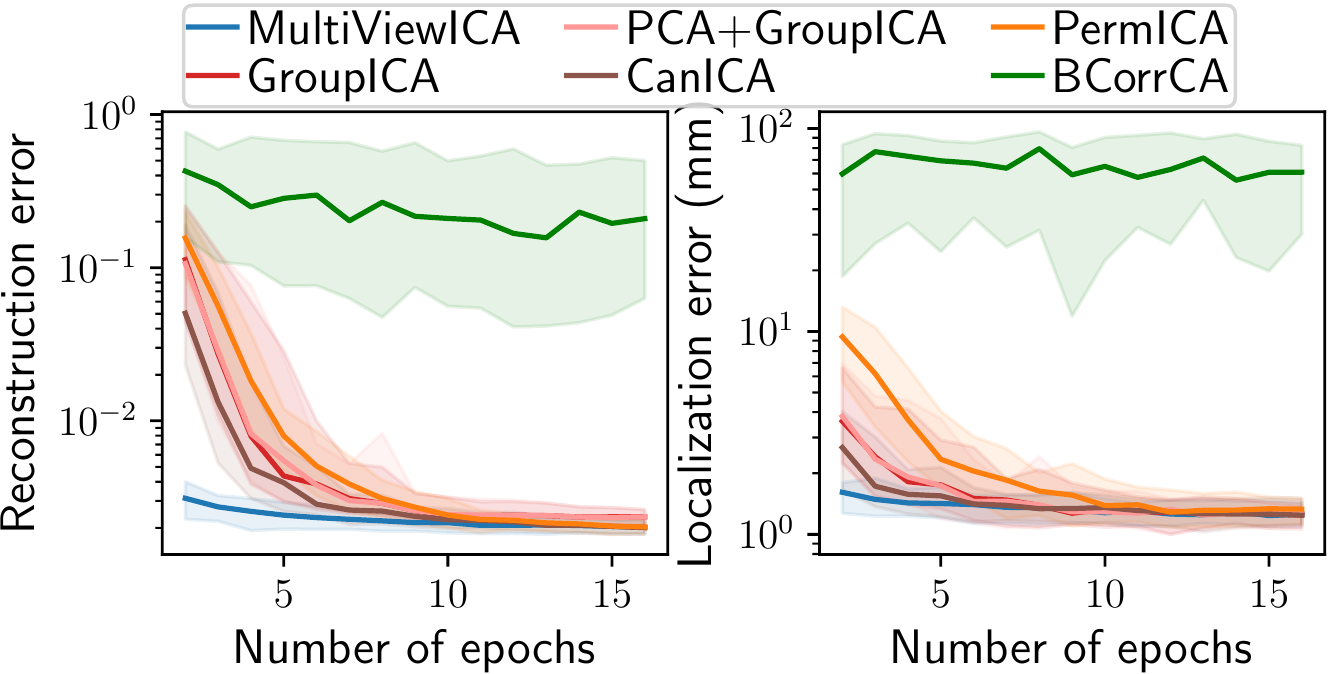}
            \caption{Phantom experiment} 
        \end{subfigure}
      \end{minipage}
      \hfill
      \begin{minipage}{0.45\linewidth}
        \begin{subfigure}[t]{\textwidth}
            \includegraphics[width=0.99\textwidth]{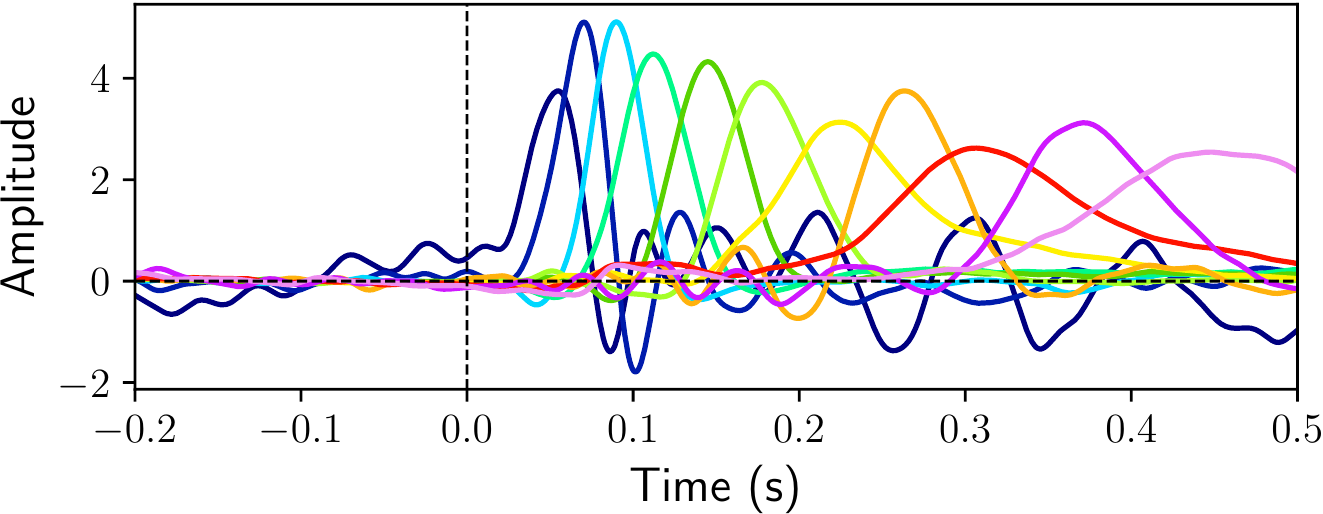}
        \end{subfigure}\\
        \begin{subfigure}[t]{\textwidth}
            \includegraphics[width=0.99\textwidth]{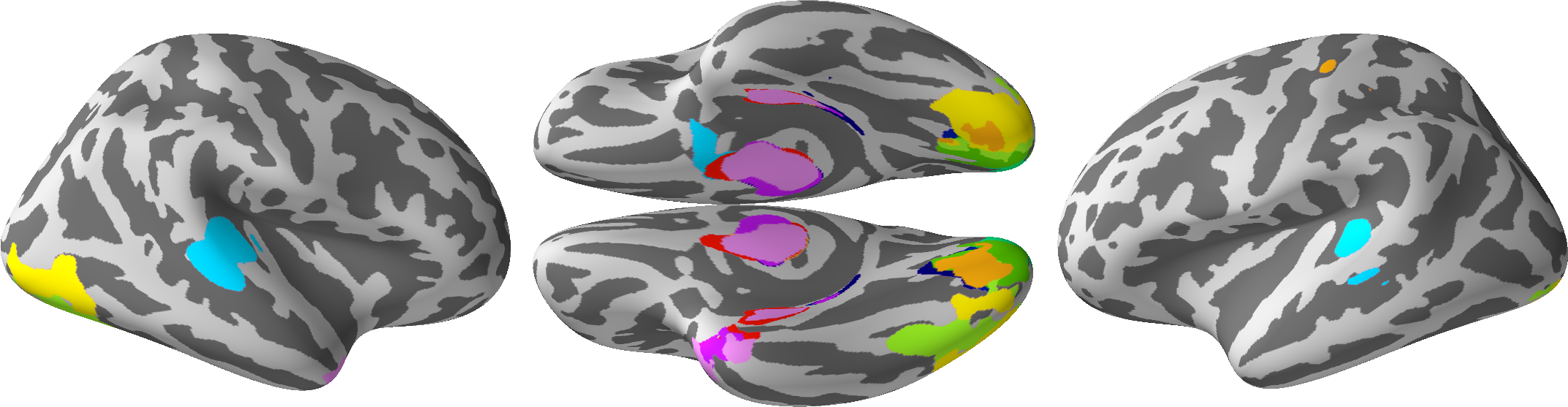}
            \caption{Cam-CAN experiment}
        \end{subfigure}
      \end{minipage}
    \end{minipage}
    \setlength{\belowcaptionskip}{-10pt}
    \caption{\emph{Left:} \textbf{Experiment on MEG Phantom data}. Reconstruction error is the norm of the difference between the estimated and true source. Localization error is the distance between the estimated and true dipole. \emph{Right:} \textbf{Experiment on 200 subjects from the CAM-can dataset} \emph{Top:} Time course of $11$ shared sources (one color per source). We recover clean evoked potentials. \emph{Bottom:} Associated brain maps, obtained by averaging source estimates registered to a common reference.}
    \label{fig:meg}
\end{figure}

\vspace{-11pt}
\section{Conclusion}
\label{sec:disc}
We have proposed a novel unsupervised algorithm that reveals latent sources
observed through different views. Using an independence assumption, 
we have demonstrated that the model is identifiable.
In contrast to previous approaches, the proposed model leads to a closed-form likelihood, which we then optimize efficiently using a dedicated alternate quasi-Newton approach.
Our approach enjoys the statistical guarantees of maximum-likelihood theory, while still being tractable.
We demonstrated the usefulness of MultiView ICA for neuroimaging group studies both on fMRI and MEG data, where it outperforms other methods.
In the experiments on fMRI data, we used temporal ICA in order to make use of the fact that subjects were exposed to the same stimuli. However, applying MultiViewICA on transposed data would carry out spatial ICA. Therefore MultiViewICA can be readily used to analyse different kind of neuroimaging data such as resting state data. 
Our method is not specific to neuroimaging data and could be relevant to other observational sciences like genomics or astrophysics where ICA is already widely used.

%
%
%
%
\section*{Broader Impact}
We develop a novel unsupervised learning method for Independent Component Analysis of a group of subjects sharing commmon sources.
Our method is not limited to a particular type of data, and could hence be employed in observational sciences where ICA is relevant: neurosciences, genomics, astrophysics, finance or computer vision for instance.
ICA is widely used in these fields as a tool among data processing pipelines, and therefore inherits from all the ethical questions of the fields above.
In particular, data collection bias will result in biased outputs.
Our algorithm is based on individual linear transforms and therefore decisions based on its application are easier to interpret than more complex models such as deep learning methods: in most applications, the set of parameters has a natural interpretation. For instance in EEG, MEG and fMRI processing, the coefficients of the linear operator can be interpreted as topographic brain maps.

\section*{Acknowledgement and funding disclosure}

This work was supported in part by the French government under management of Agence Nationale de la Recherche as part of the “Investissements d'avenir” program, references ANR19-P3IA-0001 (PRAIRIE 3IA Institute) and ANR17-CONV-0003 (DataIA Institute). It has also received funding from the European Union’s Horizon 2020 Framework Programme for Research and Innovation under the Specific Grant Agreement No. 945539 (Human Brain Project SGA3), the KARAIB AI chair (ANR-20-CHIA-0025-01) and the European Research Council grant ERC-SLAB-StG-676943. L.G. was hosted for part of this project by the Parietal team at Inria, Saclay, while on an ELLIS exchange. A.H. was additionally supported by CIFAR as a Fellow.

%
\bibliographystyle{plain}
\bibliography{multi_ica}

\clearpage

\appendix
\begin{center}
{\centering \LARGE APPENDIX}
\vspace{1cm}
\sloppy

\end{center}

\section{Likelihood}
\label{sec:app_likelihood}
\subsection{Initial form of likelihood}\label{sec:appendix:likelihood_transform}

To derive the likelihood, we start by conditioning on $\sbb$. Then, we make a variable transformation from $\xb^i$ to $\nb^i=W^i\xb^i-\sbb$, as opposed to the transformation to $\sbb$ as is usual in ICA. Using the probability transformation formula, we obtain
\begin{equation}
p(\xb^i|\sbb)=|W^i|p^i_n(W^i\xb^i-\sbb)    
\end{equation}
where $p^i_n$ is the distribution of $\nb^i$. Note that the $\xb^i$ are conditionally independent given $\sbb$, so we have their joint probability as
\begin{equation}
    p(\xb|\sbb)=\prod_{i=1}^m  |W^i| p^i_n(W^i\xb^i-\sbb)
\end{equation}
and we next get the joint probability as
\begin{equation}
    p(\xb,\sbb)=p(\sbb) \prod_{i=1}^m  |W^i| p^i_n(W^i\xb^i-\sbb)
\end{equation}
Integrating out $\sbb$ gives Eq.~(\ref{eq:likelihood}).

\subsection{Integrating out the sources}\label{sec:appendix:integration}

The integral in question, after factorization, is given by
\begin{equation}
\int_{\sbb} \prod_{j=1}^k \exp \left( -\frac{1}{2\sigma^2} \sum_{i=1}^m ((\wb_j^i)^{\top}\xb^i-s_j)^2 \right) d(s_j) d\sbb
\end{equation}
which factorizes for each $j$. Denote $y^i_j=(\wb_j^i)^{\top}\xb^i$ and $\tilde{s_j}=\frac1m\sum_{i=1}^m y^i_j$.  Fix $j$, and drop it to simplify notation. Then we need to solve the integral
\begin{align*}
   &\int_s \exp \left(-\frac{1}{2\sigma^2} \sum_{i=1}^m (y^i-s)^2 \right) d(s)ds\\
   &=\int_s \exp \left(-\frac{1}{2\sigma^2} [ m(\tilde{s}-s)^2 + \sum_{i=1}^m (y^i-\tilde{s})^2] \right) d(s)ds \\ 
&= \exp \left(-\frac{1}{2\sigma^2}\sum_{i=1}^m (y^i-\tilde{s})^2 \right) 
\int_z \exp \left(-\frac{m}{2\sigma^2} z^2 \right) d(\tilde{s}-z) dz
\end{align*}

where we have made the change of variable $z=\tilde{s}-s$. The remaining integral simply means that $d$ is smoothed by a Gaussian kernel, which can be computed exactly if $d$ is a Gaussian mixture. We therefore define $f(s) = \log \left(\int_z \exp \left(-\frac{m}{2\sigma^2} z^2 \right) d(s-z) dz\right)$.
\section{Initialization of MultiViewICA}
\label{sec:app_init}
Since the cost function $\mathcal{L}$ is non-convex, having a good initialization can make a difference in the final result.
We propose a two stage approach.
We begin by applying PermICA on the datasets, which gives us a first set of unimixing matrices $W_1^1, \dots, W_1^m$.
Note that we could also use GroupICA for this task.
Next, we perform a diagonal scaling of the mixing matrices, i.e. we find the diagonal matrices $\Lambda^1, \dots, \Lambda^m$ such that $\mathcal{L}(\Lambda^1W_1^1, \dots, \Lambda^mW_1^m)$ is minimized.
To do so, we employ Algorithm~\ref{algo:mv_ica} but only take into account the diagonal of the descent direction at each step: the update rule becomes $W^i \leftarrow (I_k + \rho \text{Diag}(D))W^i$.
The initial unmixing matrices for Algorithm~\ref{algo:mv_ica} are then taken as $\Lambda^1W_1^1, \dots, \Lambda^mW_1^m$.

Empirically, we find that this two stage procedure allows for the algorithm to start close from a satisfactory solution.
\section{Proofs of Section~\ref{sec:mvica}}
\label{sec:app_proofs}
\subsection{Proof of Prop.~\ref{prop:identifiability}}
We fix a subject $i$. Since $\sbb$ has independent components, so does $\sbb + \nb^i$. Following~\cite{comon1994independent},
Theorem 11, there exists a scale-permutation matrix $P^i$ such that $A'^i =
A^iP^i$. As a consequence, we have $\sbb  + \nb^i = P^i(\sbb' + \nb'^i)$ for all
$i$.

Then, we focus on subject 1 and subject $i \neq 1$:
\begin{align}
  &\sbb + \nb^1 - (\sbb + \nb^i) = P^1(\sbb' + \nb'^1) - P^i(\sbb' + \nb'^i)\\
  &\nb^1 - \nb^i = P^1(\sbb' + \nb'^1) - P^i(\sbb' + \nb'^i)\\
  &\iff P^1\sbb' - P^i\sbb' = P^i \nb'^i - \nb^i + \nb^1 - P^1 \nb'^1 \label{eq:condition_gaussian}
\end{align}
Since the right hand side of equation (\ref{eq:condition_gaussian}) is a linear combination of Gaussian random variables, this would imply that $P^1\sbb' - P^i\sbb'$ is also Gaussian. However, given that $\sbb'$ is assumed to be non-Gaussian, the equality can only hold if $P^1
= P^i$ and both the right and the left hand side vanish.
Therefore, the matrices $P^i$ are all equal, and there exists a scale and permutation matrix $P$ such that $A'^i = A^iP$.

\subsection{Proof of Prop.~\ref{prop:robust}}
\label{ref:robust}
We consider $W^i = \Lambda (A^i)^{-1}$, where $\Lambda$ is a diagonal matrix.
We recall $\xb^i= A^i (\sbb + \nb^i),$ so that $\yb^i = W^i\xb^i= \Lambda(\sbb + \nb^i)$.
The gradient of $\loss$ is given by eq.~\eqref{eq:gradient}:
\begin{align}
  G^i &= \frac{1}{m}f'(\tilde{\sbb})(\sbb + \nb^i)^{\top}\Lambda + \frac{1 - 1/m}{\sigma^2} \Lambda\left(\nb_i - \frac{1}{m-1}\sum_{j\neq i} \nb^j\right)(\sbb + \nb^i)^{\top}\Lambda - I_k \\
  & = \frac{1}{m}f'(\Lambda(\sbb + \frac1m\sum_j\nb^j))(\sbb+\nb^i)^{\top} \Lambda + \frac{\sigma'^2(1 - 1/m)}{\sigma^2} \Lambda^2 - I_k
\end{align}
where we write $f'(\sbb) = \begin{bmatrix}f'(s_1) \\ \vdots \\ f'(s_k) \end{bmatrix}$.
Therefore, $G^i$ is diagonal and constant across subjects (because $f'(\Lambda(\sbb + \frac1m\sum_j\nb^j))(\nb^i)^{\top} = f'(\Lambda(\sbb + \frac1m\sum_j\nb^j))(\nb^{i'})^{\top}$).
Let us therefore consider only its coefficient $(a, a)$, and let $\lambda = \Lambda_{aa}$:
$$
G^i_{aa} =G(\lambda) = \phi(\lambda)\lambda + \frac{\sigma'^2(1 - 1/m)}{\sigma^2}\lambda^2 - 1,
$$
where $\phi(\lambda) = \frac{1}{m}f'(\lambda(s_a + \frac1m\sum_j n^j_a))(s_a+n_a^i)$. One the one hand, we have $G(0) = -1$. On the other hand, if we assume for instance that $f'$ has sub linear growth (i.e. $|f'(x)| \leq c|x|^{\alpha} +d$ for some $\alpha < 1$) or that $\phi$ is positive, we find that $G(+\infty) = +\infty$.
Therefore, $G$ cancels, which concludes the proof.

\subsection{Stability conditions}
\label{sec:stability}
We consider $W^i = \Lambda (A^i)^{-1}$ where $\Lambda$ is such that the gradients $G^i$ all cancel. We consider a small relative perturbation of $W^i $ of the form $W^i \leftarrow (I_k + E^i)W^i$, and consider the effect on the gradient.
We define $\Delta^i=G^i\left((I_k + E^1)W^1, \dots, (I_k + E^m)W^m\right)$.
Denoting $C = \frac{1 - 1/ m}{\sigma^2}$ and $\tilde{\nb} = \frac1m\sum_{i=1}^m \nb^i$, we find:
\begin{align}
     &\Delta^i=\underbrace{\frac1m f'\left(\Lambda(\sbb + \tilde{\nb}) + \frac1m \sum_{j=1}^m E^j\Lambda(\sbb +\nb^j)\right)(\sbb +\nb^i)^{\top}\Lambda(I_k + E^i)^{\top} }_{\Delta_1^i} +\\
     &C\underbrace{\left(\Lambda\nb^i - \frac{1}{m-1}\sum_{j\neq i} \Lambda\nbb^j + E^i\Lambda(\sbb + \nb^i) - \frac{1}{m-1}\sum_{j\neq i} E^j \Lambda(\sbb + \nb^j)\right)(\sbb + \nb^i)^{\top}\Lambda(I_k + E^i)^{\top}}_{\Delta_2^i} \\
     &- I_k\\
\end{align}

The first term is expanded at the first order, denoting $S = \sum_{j=1}^m E^j$:

\begin{align}
    \Delta_1^i &= \frac1m \left(f'(\Lambda(\sbb + \tilde{\nb})) + f''(\Lambda(\sbb + \tilde{\nb}))\odot \left(\frac1m \sum_{j=1}^m E^j\Lambda(\sbb +\nb^j)\right)\right)(\sbb +\nb^i)^{\top}\Lambda(I_k + E^i)^{\top}\\
    &=\frac1m f'(\Lambda(\sbb + \tilde{\nb}))(\sbb + \nb^i)^{\top}\Lambda(I_k + E^i)^{\top} + \frac1{m^2}S\odot  \left(f''(\Lambda(\sbb + \tilde{\nb}))(\sbb^2)^{\top}\Lambda^2 \right)\\
    &+\frac{1}{m^2}E^i\odot\left(f''(\Lambda(\sbb + \tilde{\nb}))((\nb^i)^2)^{\top}\Lambda^2 \right)
\end{align}
The symbol $\odot$ denotes the element-wise multiplication, $f'(\sbb) = \begin{bmatrix}f'(s_1) \\ \vdots \\ f'(s_k) \end{bmatrix}$ and $f''(\sbb) = \begin{bmatrix}f''(s_1) \\ \vdots \\ f''(s_k) \end{bmatrix}$.
Similarly, the second term gives at the first order: 
\begin{align}
    \Delta_2^i &= \sigma'^2\Lambda^2(I_k + E^i)^{\top} + (1 + \sigma'^2)E^i\Lambda^2 - \frac{1}{m-1} (S - E^i) \Lambda^2
\end{align}

Combining this, we find:

\begin{align}
 \Delta^i = (E^i)^{\top} + E^i \odot\Gamma^E
 +S\odot \Gamma^S
\end{align}
where 
$$
\Gamma^E= \left(\frac1{m^2}f''(\Lambda(\sbb + \tilde{\nb}))((\nb^i)^2)^{\top} + (1-\frac1m)\frac{\sigma'^2}{\sigma^2} + \frac{1}{\sigma^2} \right)\Lambda^2
$$
$$
\Gamma^S =\left(\frac1{m^2}f''(\Lambda(\sbb + \tilde{\nb}))(\sbb^2)^{\top} -\frac{1}{m\sigma^2}  \right)\Lambda^2
$$

are $k\times k$ matrices, independent of the subject.
This linear operator is the Hessian block corresponding to the $i$-th subject:
Denoting $\mathcal{H}$ the Hessian, it is the mapping $\mathcal{H}(E^1, \dots, E^m) = (\Delta^1, \dots, \Delta^m)$.

The coefficient $\Delta^i_{ab}$ only depends on $(E^i_{ab}, E^i_{ba}, E^1_{ab},\dots, E^m_{ab})$. Therefore, the Hessian is block diagonal with respect to the blocks of coordinates $(E^1_{ab}, E^1_{ba}, \dots, E^m_{ab}, E^m_{ba})$. Denote $\varepsilon = \Gamma^E_{ab}$, $\varepsilon' = \Gamma^E_{ba}$, $\beta = \Gamma^S_{ab}$ and $\beta'= \Gamma^S_{ba}$. The linear operator for the block is:

$$
K(\varepsilon, \varepsilon', \beta, \beta')=
\left(
    \begin{array}{ll|ll|l|ll}
\varepsilon + \beta & 1       & \beta & 0       & \dots  & \beta & 0       \\
1      & \varepsilon' + \beta' & 0      & \beta' & \dots  & 0      & \beta' \\
\hline
\beta & 0       & \varepsilon + \beta & 1       &        & \beta & 0       \\
0      & \beta' & 1      & \varepsilon' + \beta' & \ddots & 0      & \beta' \\
\hline
\vdots & \vdots  &        & \ddots  & \ddots & \vdots & \vdots  \\
\hline
\beta & 0       & \beta & 0       & \dots  & \varepsilon + \beta & 1       \\
0      & \beta' & 0      & \beta' & \dots  & 1      & \varepsilon' + \beta'
    \end{array}
\right)
$$
The positivity of $\mathcal{H}$ is equivalent to the positivity of this operator for all pairs $a, b$.
We now assume $\beta \beta' > 0$.

First, we should note that $K(\varepsilon, \varepsilon', \beta, \beta') $ is congruent to $K(\varepsilon \sqrt{\frac{\beta'}{\beta}}, \varepsilon' \sqrt{\frac{\beta}{\beta'}}, \sqrt{\beta\beta'}, \sqrt{\beta\beta'})$ via the basis $\text{diag}((\frac{\beta'}{\beta})^{1/4}, (\frac{\beta}{\beta'})^{1/4}, \cdots,(\frac{\beta'}{\beta})^{1/4}, (\frac{\beta}{\beta'})^{1/4})$.
We denote to simplify notation $\alpha = \varepsilon \sqrt{\frac{\beta'}{\beta}}$, $\alpha' = \varepsilon' \sqrt{\frac{\beta}{\beta'}}$ and $\gamma = \sqrt{\beta\beta'}$. We only have to study the positivity of $K(\alpha, \alpha', \gamma, \gamma)$.
We have:
$$
K(\alpha, \alpha', \gamma, \gamma) =  I_m  \otimes M_\alpha+ \gamma  \mathbb{1}\otimes I_2, \enspace M_\alpha = 
\begin{pmatrix}
\alpha & 1 \\
1 & \alpha'
\end{pmatrix}
$$
Since $I_m\otimes M_\alpha$ and $\gamma \mathbb{1}\otimes I_2$ commute, the minimum value of $\text{Sp}(K)$ is $\text{min}(I_m\otimes M_\alpha) + \text{min}(\gamma\text{Sp}(\mathbb{1}))=\frac12(\alpha + \alpha' - \sqrt{(\alpha - \alpha')^2 + 4}) + m\min(0, \gamma)$.
Since we assumed $\beta \beta' > 0$ we have $\gamma > 0$. This is similar to the usual ICA case, we find that the condition is $\alpha\alpha' > 1$.

If the following conditions hold for all pair of sources $a, b$, the sources are a local minimum of the cost function:
\begin{itemize}
    \item $\Gamma^S_{ab}\Gamma^S_{ba}\geq 0$
    \item $\Gamma^E_{ab}\Gamma^E_{ba} > 1$
\end{itemize}
\section{Identifiability for Shared Response Model}
\label{sec:app_identifiability}
The shared response model~\cite{chen2015reduced} (SRM) models the data $\xb^i \in \bbR^v$ of subject $i$ for $i = 1,\dots, m$ as
\begin{align*}
    \xb^i = A^i \sbb + \nb^i \enspace \text{with} \enspace \sbb \sim \Ncal(0, \Sigma), \enspace\nb^i \sim \Ncal(0, \rho_i^2 I_v), \enspace {A^i}^{\top}A^i = I_k
\end{align*}
where $A^i \in \bbR^{v, k}$, $\sbb \in \bbR^k$ and  $\Sigma \in \mathbb{R}^{k, k}$ is a symmetric positive definite matrix.

\begin{proposition}
SRM is not identifiable
\end{proposition}
\begin{proof}
Let us assume the data $\xb^i \enspace i=1, \dots, m$ follow the SRM model with parameters $\Sigma, A^i, \rho_i^2 \enspace i=1, \dots, m$. 

Let us consider an orthogonal matrix $O \in \Ocal_k$.
We call $A'^i = A^i O$ and $\Sigma' = O^{\top} \Sigma O$. 
$\Sigma'$ is trivially symmetric positive definite.

Then the data also follows the SRM model with different parameters $\Sigma', A'^i, \rho_i^2 \enspace i=1, \dots, m$.
\end{proof}

\begin{proposition}
We consider the decorrelated SRM model with an additional decorrelation assumption on the shared responses.
\begin{align*}
\xb^i = A^i \sbb + \nb^i \enspace \text{with} \enspace \sbb \sim \Ncal(0, \Sigma), \enspace\nb^i \sim \Ncal(0, \rho_i^2 I_v), \enspace {A^i}^{\top}A^i = I_k
\end{align*}
where $\Sigma$ is a positive \emph{diagonal} matrix. We further assume that the values in $\Sigma$ are all distinct and ranked in ascending order.
The decorrelated SRM is identifiable up to sign indeterminacies on the columns of 
$\begin{bmatrix}
A^1 \\
\vdots \\
A^m
\end{bmatrix}
$.
\end{proposition}
\begin{proof}
The decorrelated SRM model can be written
\begin{align*}
    &\xb^i \sim \Ncal(0, A^i \Sigma {A^i}^{\top} + \rho_i^2 I_v) \enspace \text{with}\enspace  {A^i}^{\top}A^i = I_k
\end{align*}
where $\Sigma$ is a positive diagonal matrix with distincts values ranked in ascending order.

Let us assume the data $\xb^i \enspace i=1, \dots, m$ follow the decorrelated SRM model with parameters $\Sigma, A^i, {\rho_i}^2 \enspace i=1, \dots, m$. Let us further assume that the data $\xb^i \enspace i=1, \dots, m$ follow the decorrelated SRM model with an other set of parameters $\Sigma', A'^i, {\rho'_i}^2 \enspace i=1, \dots, m$.

Since the model is Gaussian, we look at the covariances.
We have for $i \neq j$
\begin{align*}
    \bbE[\xb^i\left(\xb^j\right)^{\top}] = A^i\Sigma {A^j}^{\top} = A'^i \Sigma'{A'^j}^{\top} \enspace, 
\end{align*}
The singular value decomposition is unique up to sign flips and permutation. Since eigenvalues are positive and ranked the only indeterminacies left are on the eigenvectors. For each eigenvalue a sign flip can occur simultaneously on the corresponding left and right eigenvector.

Therefore we have $\Sigma' = \Sigma$, $A^i = A'^i D^{ij}$ and $A^j = A'^j D^{ij}$ where $D^{ij} \in \bbR^{k, k}$ is a diagonal matrix with values in $\{-1, 1\}$. This analysis holds for every $j \neq i$ and therefore $D^{ij} = D$ is the same for all subjects.

We also have for all $i$
\begin{align*}
    \bbE[\xb^i \left(\xb^i\right)^{\top}] = A^i \Sigma {A^i}^{\top} + \rho_i^2 I_v =  A'^i \Sigma' {A'^i}^{\top}  + {\rho'}_i^2 I_v\\
\end{align*}
We therefore conclude ${\rho'}_i^2 = \rho_i^2, i=1 \dots m$.

Note that if the diagonal subject specific noise covariance $\rho_i^2 I_v$ is replaced by any positive definite matrix, the model still enjoys identifiability.
\end{proof}

\section{fMRI experiments}
\label{sec:app_expts}
\subsection{Dataset description and preprocessing}
\label{preprocessing}
The full brain mask used to select brain regions is available in the Python package associated with the paper.

\paragraph{Sherlock}
In \emph{sherlock} dataset, 17 participants are watching "Sherlock" BBC TV show (beginning of episode 1). 
These data are downloaded from \url{http://arks.princeton.edu/ark:/88435/dsp01nz8062179}. 
Data were acquired using a 3T scanner with an isotropic spatial resolution of 3 mm. 
More information including the preprocessing pipeline is available in~\cite{chen2017shared}.
Subject 5 is removed because of missing data leaving us with 16 participants.
Although \emph{sherlock} data are downloaded as a temporal concatenation of two runs, we split it manually into 4 runs of 395 timeframes and one run of 396 timeframes so that we can perform 5 fold cross-validation in our experiments.

\paragraph{FORREST}
In FORREST dataset 20 participants are listening to an audio version of the Forrest Gump  movie.
FORREST data are downloaded from OpenfMRI~\cite{poldrack2013toward}. 
Data were acquired using a 7T scanner with an isotropic spatial resolution of 1~mm (see more details in~\cite{hanke2014high}) and resampled to an isotropic spatial resolution of 3~mm.
More information about the forrest project can be found at \url{http://studyforrest.org}.
Subject 10 is discarded because not all runs available for other subjects were available for subject 10 at the time of writing.
Run 8 is discarded because it is not present in most subjects.
 
\paragraph{RAIDERS}
In RAIDERS dataset, 11 participants are watching the movie "Raiders of the lost ark".
The RAIDERS dataset belongs to the Individual Brain Charting dataset~\cite{ibc}.
Data were acquired using a 3T scanner and resampled to an isotropic spatial resolution of 3~mm.
The RAIDERS dataset reproduces the protocol described in~\cite{haxby2011common}.
Preprocessing details are described in~\cite{ibc}.

\paragraph{CLIPS}
In CLIPS dataset, 12 participants are exposed to short video clips. 
The CLIPS dataset also belongs to the Individual Brain Charting dataset (\cite{ibc}).
Data were acquired using a 3T scanner and resampled to an isotropic spatial resolution of 3~mm.
It reproduces the protocol of original studies described in \cite{nishimoto2011reconstructing} and \cite{huth2012continuous}.
Preprocessing details are described in~\cite{ibc}.

At the time of writing, the CLIPS and RAIDERS dataset from the individual brain charting dataset \url{https://project.inria.fr/IBC/} are available at \url{https://openneuro.org/datasets/ds002685}.
Protocols on the visual stimuli presented are available in a dedicated repository on Github: \url{https://github.com/hbp-brain-charting/public_protocols}.

\subsection{Reconstructing the BOLD signal of missing subjects: Discussion on ROIs choice}
\label{brainmaps}

The quality of the reconstructed BOLD signal varies depending on the choice of the region of interest. In Figure~\ref{fig:brainmaps}, we plot for GroupICA, SRM and MultiViewICA, the R2 score per voxel using 50 components for datasets \emph{sherlock}, \emph{forrest}, \emph{raiders} and \emph{clips}. As could be anticipated from the task definition, \emph{forrest} obtains high reconstruction accuracy in the auditory cortices, while \emph{clips} shows good reconstruction in the visual cortex (occipital lobe mostly); the richer \emph{sherlock} and \emph{raiders} datasets yield good reconstructions in both domains, but also in other systems (language, motor).
We also see visually see that data reconstructed by MultiViewICA are
a better approximation of the original data than other methods.
This is particularly obvious for the \emph{clips} datasets where it is
clear that voxels in the posterior part of the superior
temporal sulcus are better recovered by MultiViewICA than by SRM or
GroupICA.

In order to determine the ROIs, we focus on the R2 score per voxel between the BOLD signal reconstructed by GroupICA and the actual bold signal. We run GroupICA with $10, 20$ and $50$ components and select the voxels that obtained a positive R2 score for all sets of components.
We discard voxels with an R2 score above 80\% as they visually correspond to artefacts and apply a binary opening using a unit cube as the structuring element. The chosen regions are plotted in figure~\ref{fig:roi}.

\begin{figure}
  \centering
  \includegraphics[width=\textwidth]{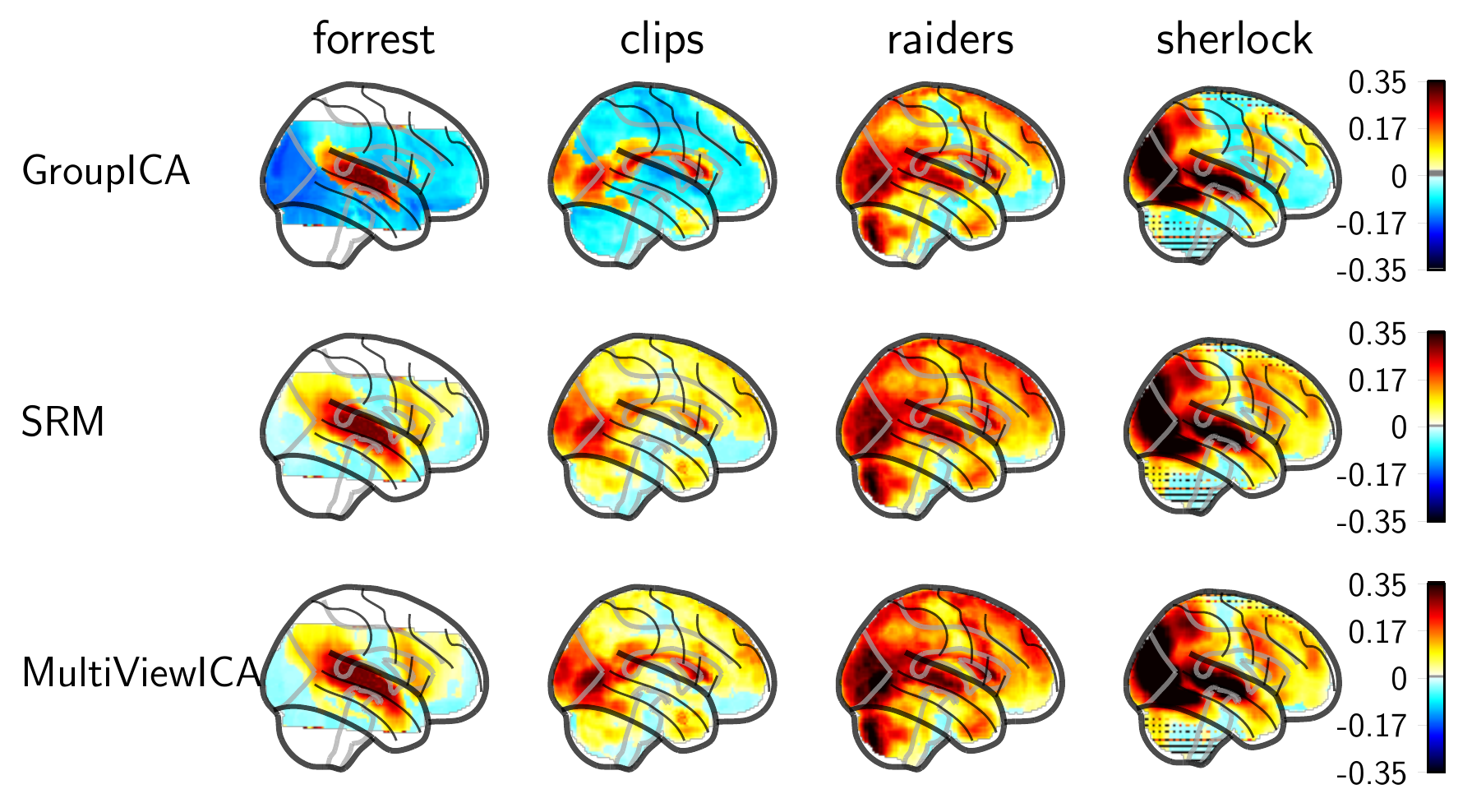}
  \caption{\textbf{Reconstructing the BOLD signal of missing subjects: Reconstruction R2 score per voxel} We plot for GroupICA, SRM and MultiViewICA, the R2 score per voxel using 50 components for datasets \emph{sherlock}, \emph{forrest}, \emph{raiders} and \emph{clips}. We visually see that data reconstructed by MultiViewICA are more faithful reproduction of the original data than other methods.}
  \label{fig:brainmaps}
\end{figure}

\begin{figure}
  \centering
  \includegraphics[width=\textwidth]{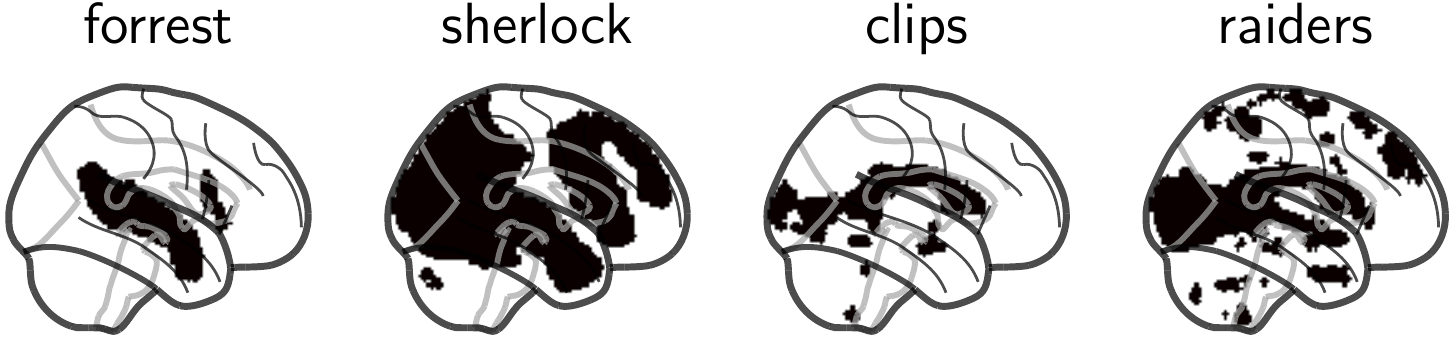}
  \caption{\textbf{Data-driven choice of ROI} Chosen ROIs for the experiment: Reconstructing the BOLD signal of missing subjects.}
  \label{fig:roi}
\end{figure}

\subsection{Between-runs time-segment matching}
\label{app_spatialmaps}

\begin{figure}
  \centering
  \includegraphics[width=\textwidth]{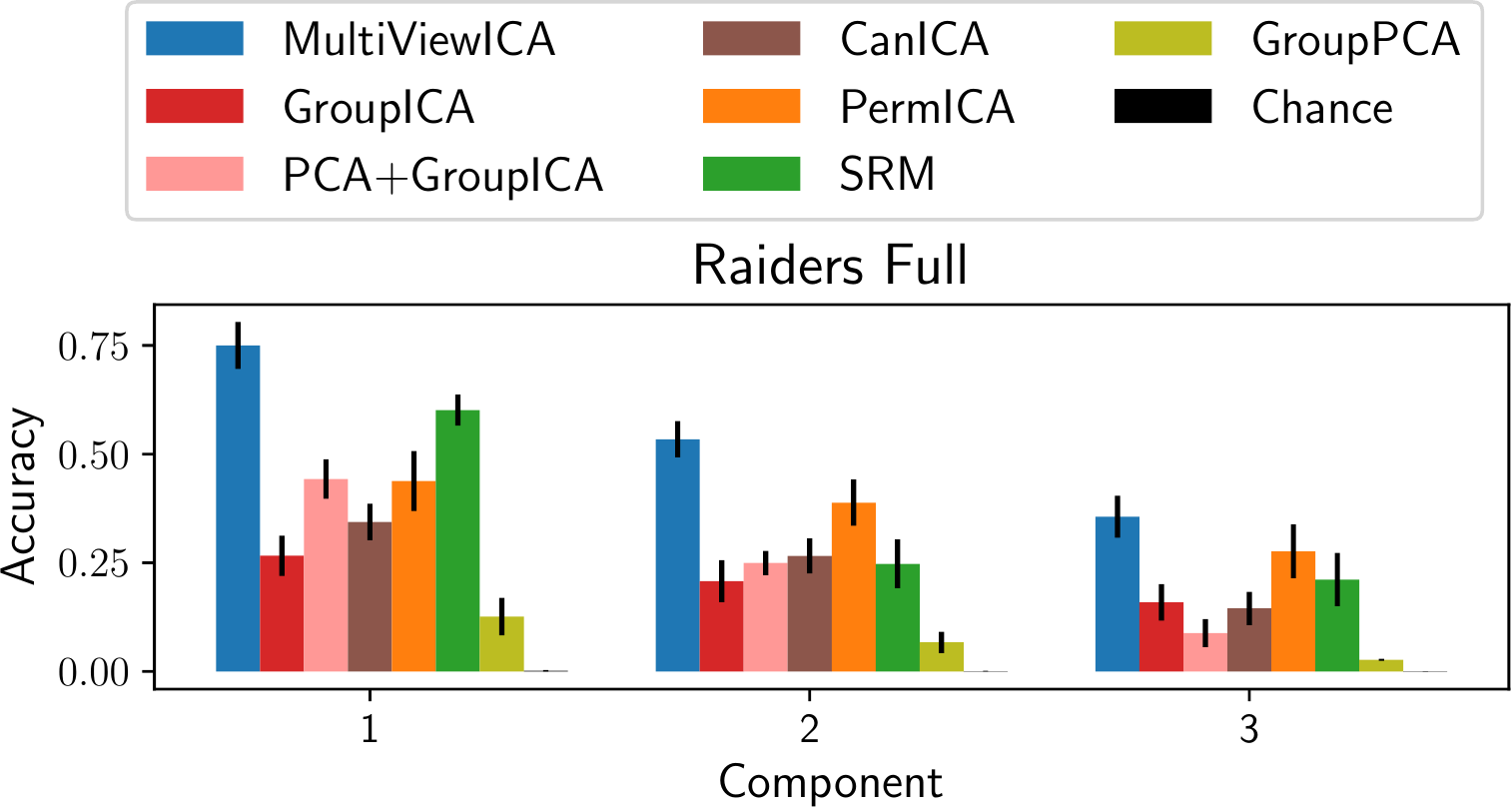}
  \caption{\textbf{Between runs time-segment matching}. Interesting sources correlates more when they correspond to the same stimulus (same scenes of the movie) than when they correspond to distinct stimuli (different scenes).
  We extract 20 sources and report the mean accuracy of the 3 best performing sources}
  \label{fig:swetha}
\end{figure}

We measure the ability of each algorithm to extract meaningful shared sources that correlate more when they correspond to the same stimulus than when they correspond to distinct stimuli. We use the \emph{raiders-full} dataset, which allows this kind of analysis because subjects watch some selected scenes from the movie twice, during the first two runs (1 and 2) and the last two (11 and 12).
First, the forward operators are learned by fitting each algorithm with 20 components on the data of all 11 subjects using all 12 runs. We then select a subset of 8 subjects and the shared sources are computed by applying the forward operators and averaging.
We select a large target time-segment ($50$
timeframes) taken at random from run 1 and 2, and we try to localize the corresponding sample time-segment from the 10 last runs using a single component of the shared sources.
The time-segment is said to be
correctly classified if the correlation between the target and corresponding sample
time-segment is higher than with any other time-segment (partially overlapping windows are excluded).
In contrast to the \emph{between subject time-segment matching} experiment, we obtain one accuracy score per component.
We repeat the experiment 10 times with different subsets of subjects randomly chosen and report the mean accuracy of the 3 best performing components in Figure~\ref{fig:swetha}. Error bars correspond to a 95~\% confidence interval.
MultiView ICA achieves the highest accuracy.

We then focus on the 3 best performing components of MultiView ICA. For each component, we plot in Figure~\ref{fig:app_spatialmaps} (left) the shared sources during two sets of runs where subjects were exposed to the same scenes of the movie. We then study the localisation of these sources.
We average the forward operators across subjects and plot the columns corresponding to the components of interest in Figure~\ref{fig:app_spatialmaps} (right).
As each column is seen as a set of weights over all voxels, it represents a spatial map.

The component 1 of the shared responses follows almost the same pattern in the two set of runs corresponding to the same scenes of the movie. The spatial map corresponding to component 1 highlights the language network.
In component 2, the temporal patterns during the viewing of identical scenes are also very similar. The corresponding spatial map highlights the visual network especially the visual dorsal pathway.
In component 3, there exists a similarity however less striking than with the two previous components. The corresponding spatial map highlights a contrast between the spatial attention network and the auditory network.

\begin{figure}
  \centering
  \includegraphics[width=\textwidth]{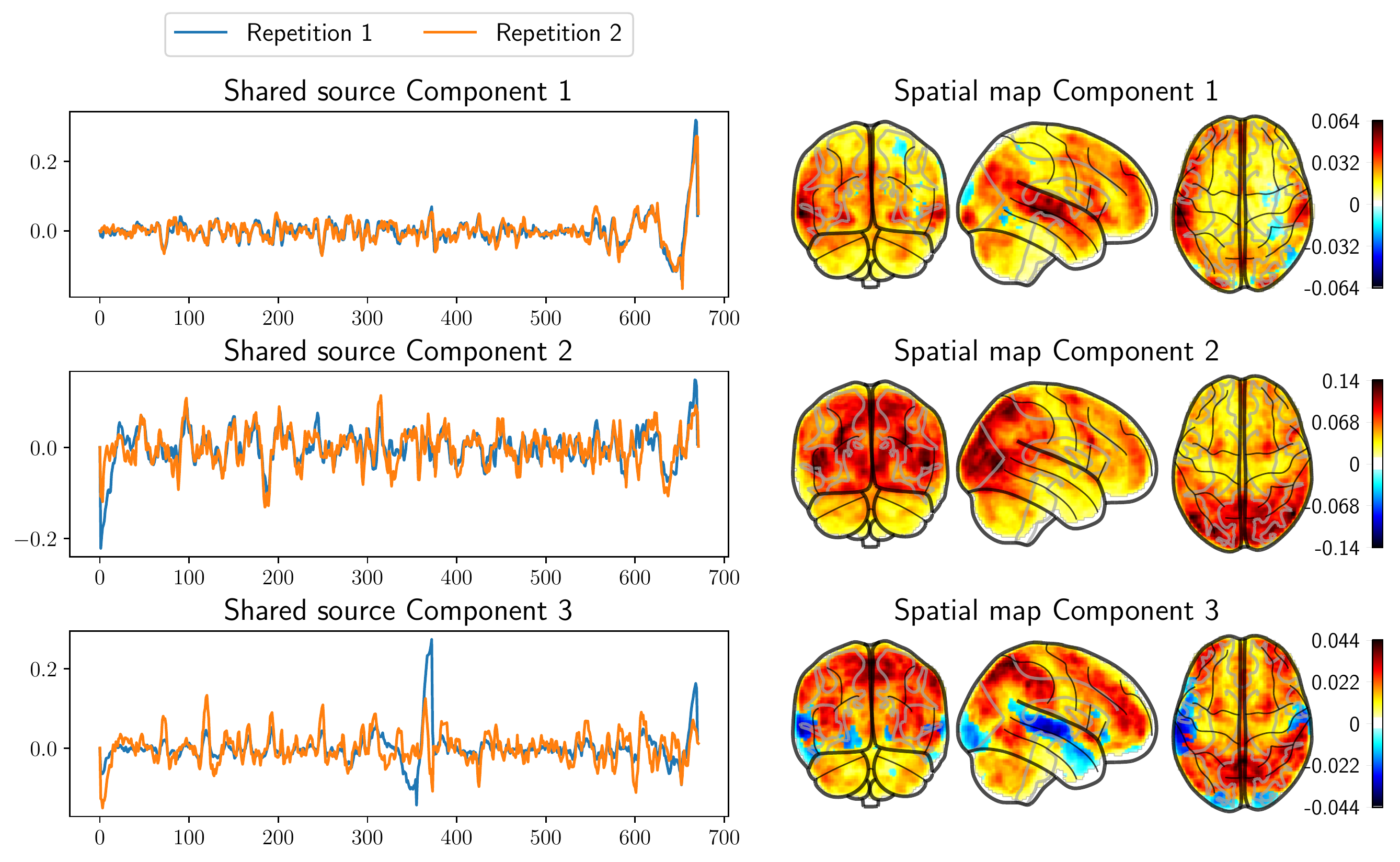}
  \caption{\textbf{Between-runs time segment matching: spatial maps and timecourses} \emph{Left:} Timecourses of the 3 shared sources yielding the highest accuracy. The two displayed set of runs correspond to the same scenes in the movie. \emph{Right:} Localisation of the same shared sources in the brain}
  \label{fig:app_spatialmaps}
\end{figure}

\subsection{Reproducing time-segment matching experiment}
\label{appendix_reproduce}
We reproduce the time-segment matching experiments described in \cite{chen2016convolutional} and \cite{zhang2016searchlight} and use two fold classification over runs instead of 5-fold as we have done in the main paper. We used the sherlock data available at \url{http://arks.princeton.edu/ark:/88435/dsp01nz8062179} and the full brain mask provided in the Python package associated with the paper. We applied high-pass filtering (140 s cutoff) and the time series of each voxel were normalized to zero mean and unit variance.

The results are available in Figure~\ref{fig:supp_timesegment}.

\begin{figure}[H]
  \centering
  \includegraphics[width=0.6\textwidth]{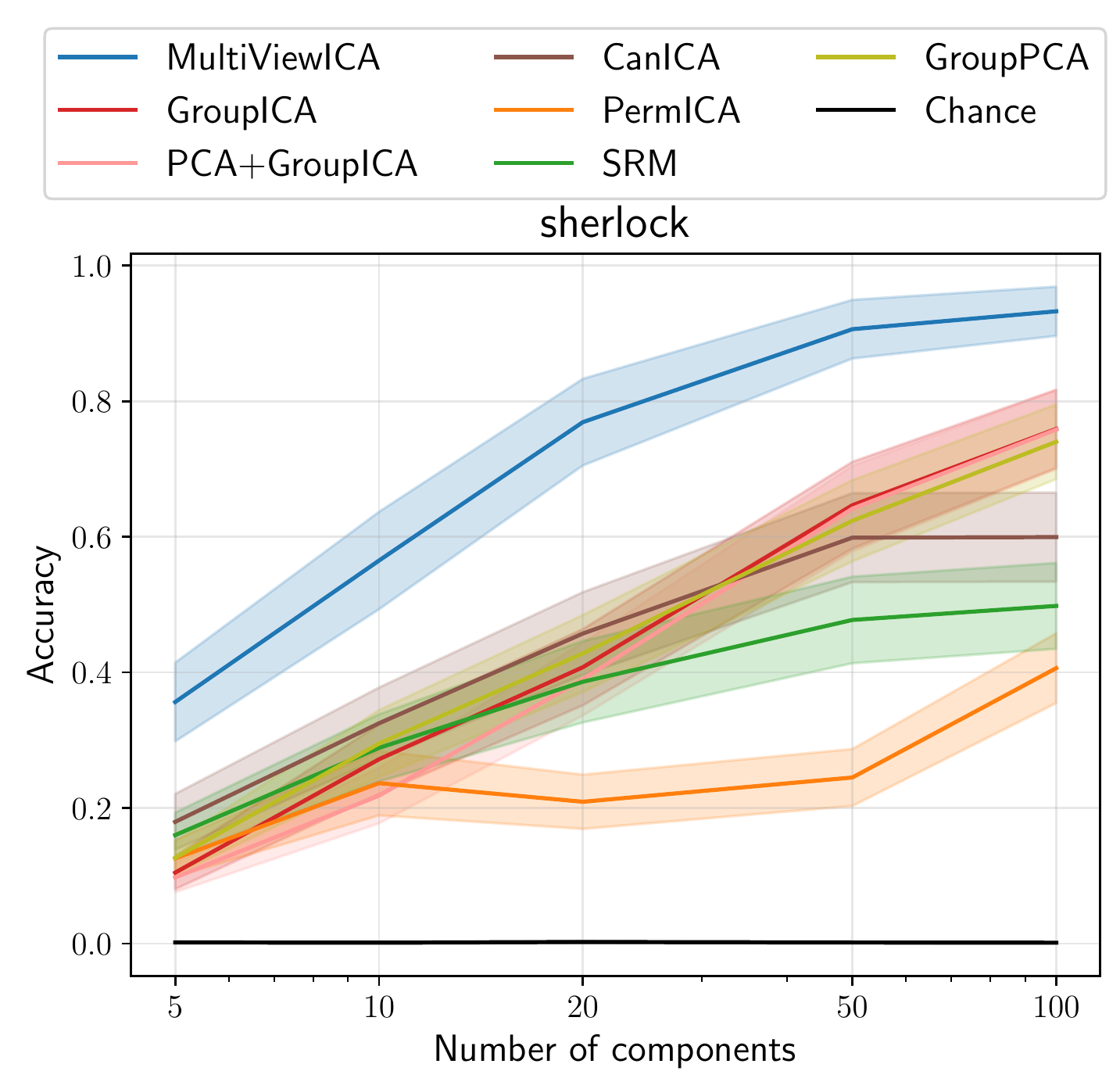}
  \caption{\textbf{Reproducing the time-segment matching experiment of \cite{chen2016convolutional}~\cite{zhang2016searchlight}} Mean classification accuracy - error bars represent 95\% confidence interval}
  \label{fig:supp_timesegment}
\end{figure}

\subsection{Impact of the hyperparameter $\sigma$ }
\label{sec:app_sigma_impact}
On top of the theoretical guarantees about the robustness of our method to the choice of the $\sigma$ parameter, we investigate its practical impact on the time-matching segment experiment, on the Sherlock dataset with $10$ components.
We compute the accuracy of the multi-view ICA pipeline with different choice of $\sigma$.
This is reported in Fig.~\ref{fig:supp_noise_sensitivity}. 
The accuracy is constant for a wide range of $\sigma$, only decreasing when $\sigma$ attains very high values.
\begin{figure}[H]
  \centering
  \includegraphics[width=0.6\textwidth]{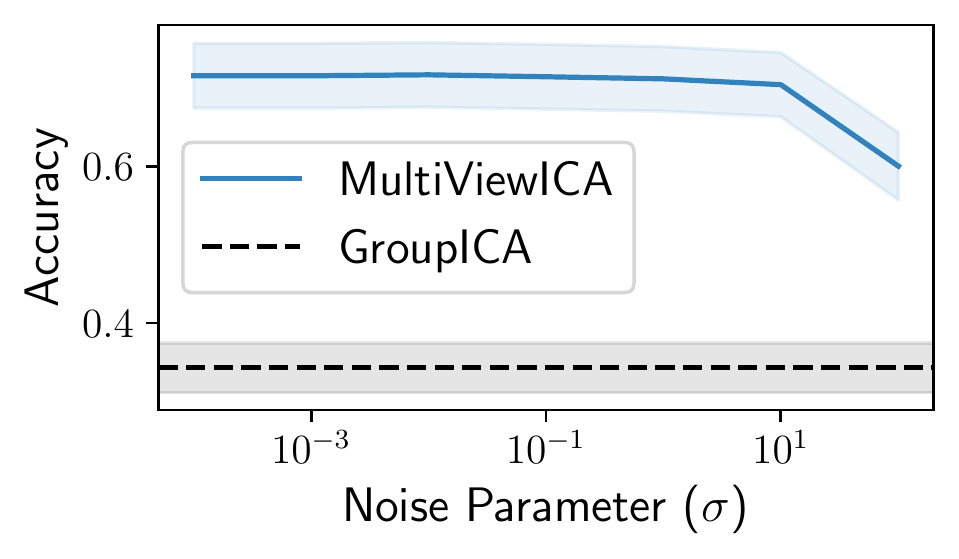}
  \caption{\textbf{Effect of the parameter $\sigma$}: We compute the accuracy of the multiview-ICA pipeline on the time-segment matching experiment for various values of the $\sigma$ hyperparameter over a grid. The accuracy varies only marginally with $\sigma$.}
  \label{fig:supp_noise_sensitivity}
\end{figure}

\section{Related Work}
\label{sec:app_rel_work}

The following table describes some usual method for extracting shared sources from multiple subjects datasets.
The column "Modality/Source" describes the type of data for which each algorithm was \emph{initially} proposed, even though each algorithm could be applied on any type of data. 
The source type can be either temporal if extracted sources are time courses or spatial if they are spatial patterns. 
\begin{center}
\begin{longtable}{ |p{.15\textwidth} | p{.2\textwidth} |p{.2\textwidth}| p{.3\textwidth}|}
\hline
\textbf{Method} & \textbf{Modality/Source} &\textbf{Dimension reduction} & \textbf{Description}  \\
\hline
SRM \cite{chen2015reduced} & 
fMRI/Temporal
&
SRM
&
The model is $\xb^i = A^i\sbb + \nb^i$, with \emph{Gaussian} sources and \emph{orthogonal} mixing matrices $A^i$\\
\hline
GroupPCA~\cite{smith2014group} &
fMRI/Spatial
&
GroupPCA
&
A memory efficient implementation of PCA applied on temporally concatenated data.\\
\hline
GIFT~ \cite{calhoun2001method} & 
fMRI/Spatial
&
Individual PCA + Group PCA (on component-wise concatenated data)
&
Single-subject ICA is applied on the aggregated data\\
\hline
EEGIFT~ \cite{eichele2011eegift} & 
EEG/Temporal
&
Individual PCA + Group PCA (on component-wise concatenated data)
&
Single-subject ICA is applied on the aggregated data\\
\hline
PermICA &
Any
&
Any
&
Single-subject ICA is applied on each subject's data, and the components are matched using the Hungarian algorithm\\
\hline
Clustering approach~\cite{esposito2005independent}&
fMRI/Spatial
&
Individual PCA
&
Single-subject ICA is applied on each subject's data, and the components are matched using a hierarchical clustering algorithm.\\
\hline
Measure projection analysis~\cite{bigdely2013measure}&
EEG/Temporal
&
Individual PCA
&
Single-subject ICA is applied on each subject's data, and the components are matched using a hierarchical clustering algorithm.\\
\hline
TensorICA \cite{beckmann2005tensorial} &
fMRI/Spatial
&
Group PCA (on spatially concatenated data)
&
TensorICA incorporates ICA assumptions into the PARAFAC model. The mixing matrices $A_1 \cdots A_n$ are such that $A_i = A D_i$ where $A$ is common to all subjects and $D_i$ are subject specific diagonal matrices.\\  
\hline
Unifying Approach of \cite{guo2008unified} &
fMRI/Spatial
&
Group PCA (on spatially concatenated data) + GroupPCA (on component-wise concatenated data).
&
The model is $\xb^i = A^i\sbb + \nb^i$ with a Gaussian mixture model on independent sources and a matrix normal prior on the noise. \\
\hline
SR-ICA \cite{zhang2016searchlight} &
fMRI/Temporal
&
SR-ICA
&
SR-ICA incorporates ICA assumptions into the shared response model.  \\
\hline
CAE-SRM \cite{chen2016convolutional}
&
fMRI/Temporal
&
CAE-SRM
&
A convolutional auto-encoder is used to perform the unmixing. \\  
\hline
CanICA \cite{varoquaux2009canica} &
fMRI/Spatial
&
Individual PCA + multi set CCA (on component-wise concatenated data)
&
CanICA applies single-subject ICA on data reduced with PCA and CCA.
 \\  
\hline
Spatial ConcatICA~\cite{svensen2002ica} &
fMRI/Spatial
&
Group PCA (on spatially concatenated data)
&
ICA is applied on spatially concatenated data. The mixing is constrained to be the same across all subjects.
 \\  
 \hline
Temporal ConcatICA~\cite{cong2013validating} &
EEG/Temporal
&
Group PCA (on temporally concatenated data)
&
ICA is applied on temporaly concatenated data. The mixing is constrained to be the same across all subjects.
 \\  
\hline
coroICA \cite{pfister2019robustifying} &
Any
&
Any
&
The model is  $\xb^i = A\sbb_i + \nb^i$. The mixing is constrained to be the same across all subjects. \\
\hline
\end{longtable}
\end{center}
\vspace{-10mm}
An additional related model is described in~\cite{gresele2019incomplete}. Similarly to our work, the ICA model has noise on the source side. However, the model involves nonlinear mixings, which are computationally unfeasible to optimize via maximum likelihood; a contrastive learning scheme is therefore adopted, and the likelihood is not derived in closed form. No evaluation on neuroimaging datasets is presented.

\section{Detailed Cam-CAN sources}
\label{sec:app_montages}
We display each of the 11 shared sources found by Multiview ICA on the Cam-CAN. The time-courses are on the left, the corresponding brain maps are on the right.

{\centering
\includegraphics[width=0.3\textwidth]{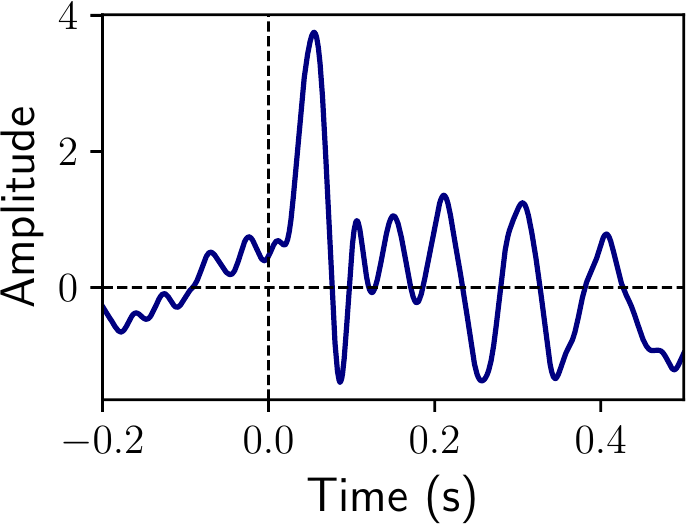}%
\raisebox{0.2\height}{\includegraphics[width=0.68\textwidth]{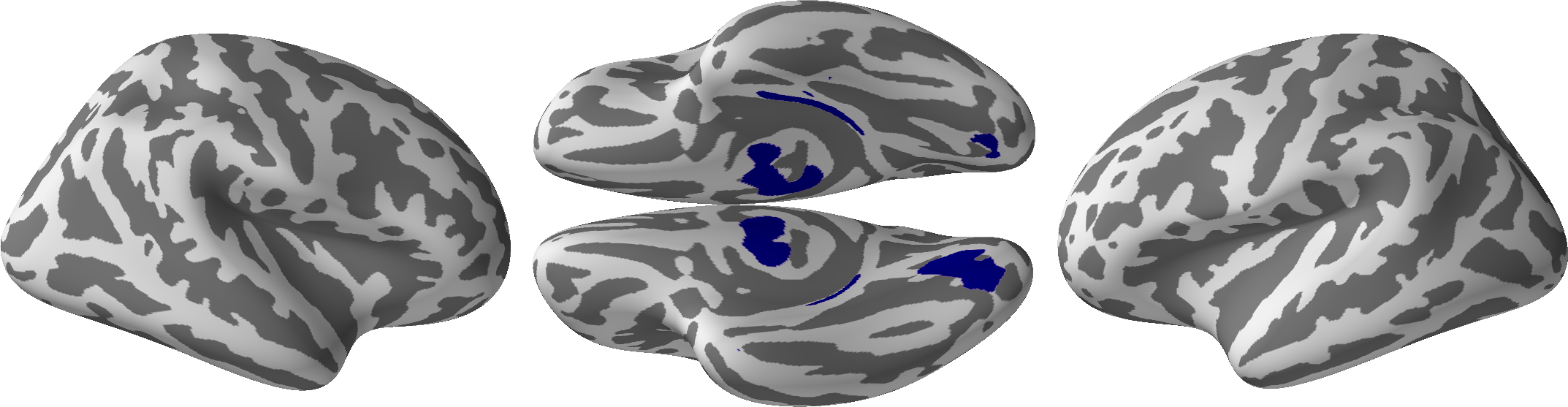}} \\
\includegraphics[width=0.3\textwidth]{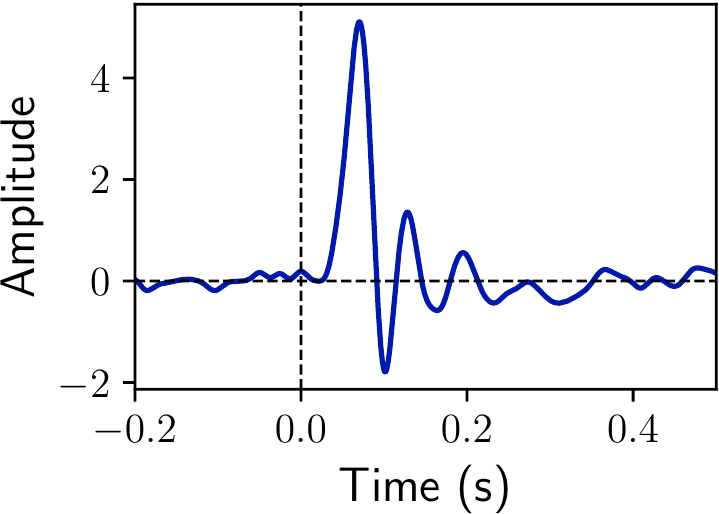}%
\raisebox{0.2\height}{\includegraphics[width=0.68\textwidth]{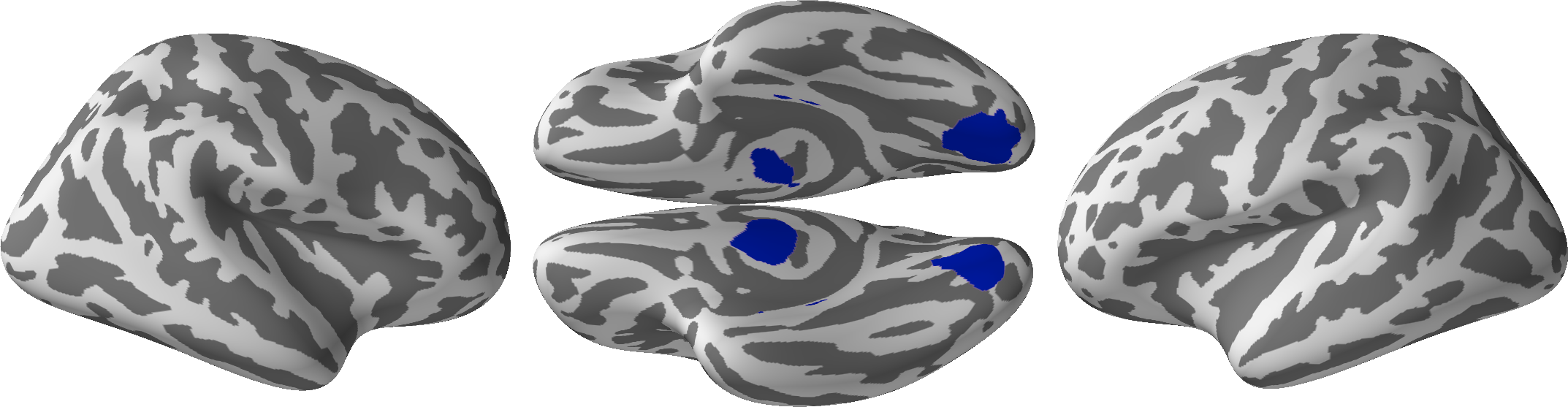}} \\
\includegraphics[width=0.3\textwidth]{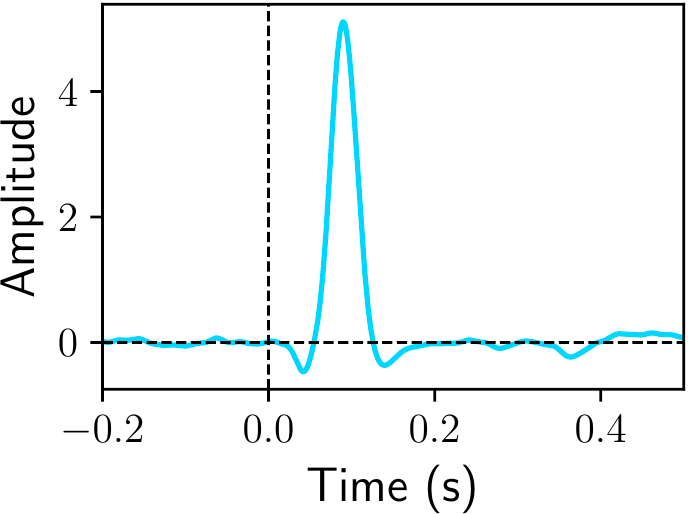}%
\raisebox{0.2\height}{\includegraphics[width=0.68\textwidth]{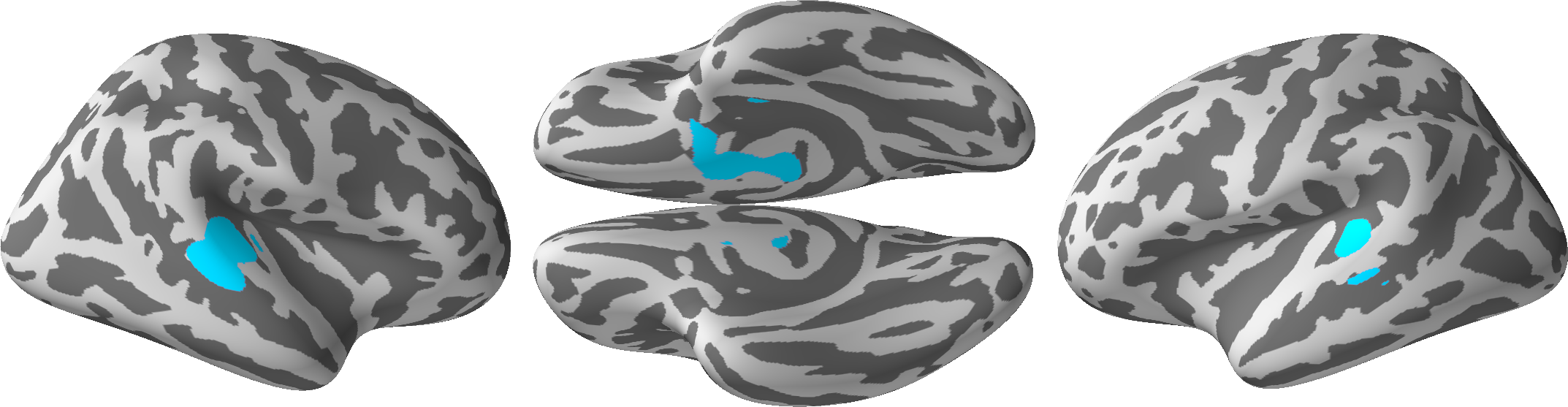}} \\
\includegraphics[width=0.3\textwidth]{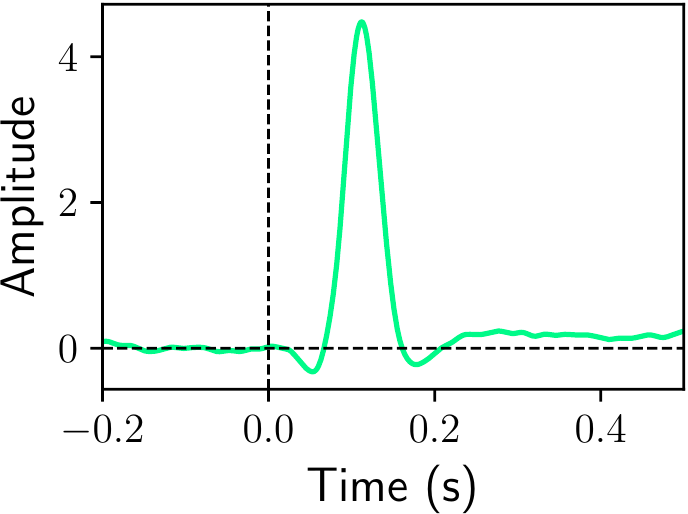}%
\raisebox{0.2\height}{\includegraphics[width=0.68\textwidth]{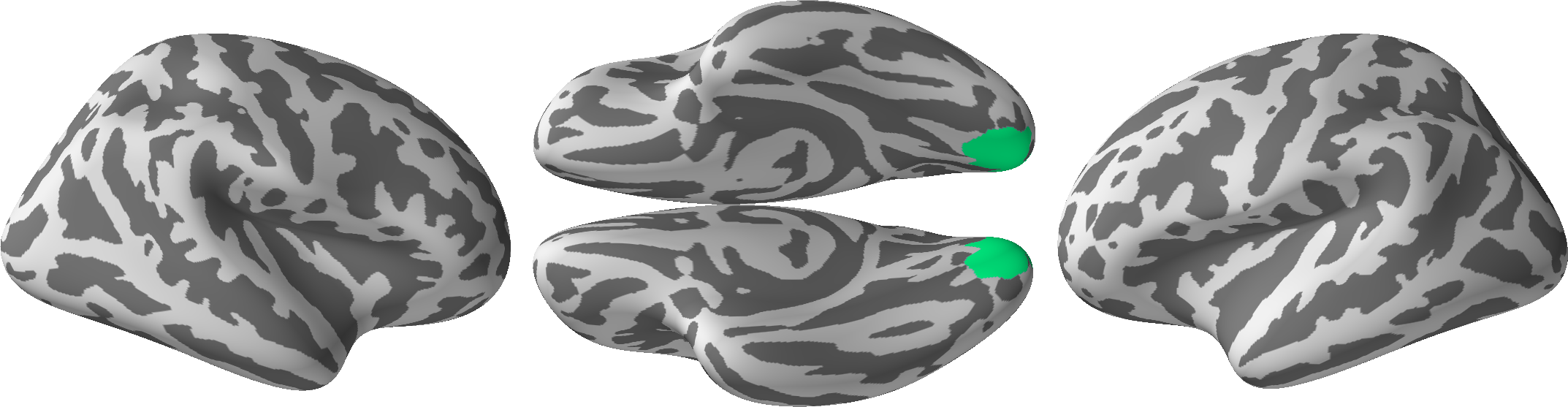}} \\
\includegraphics[width=0.3\textwidth]{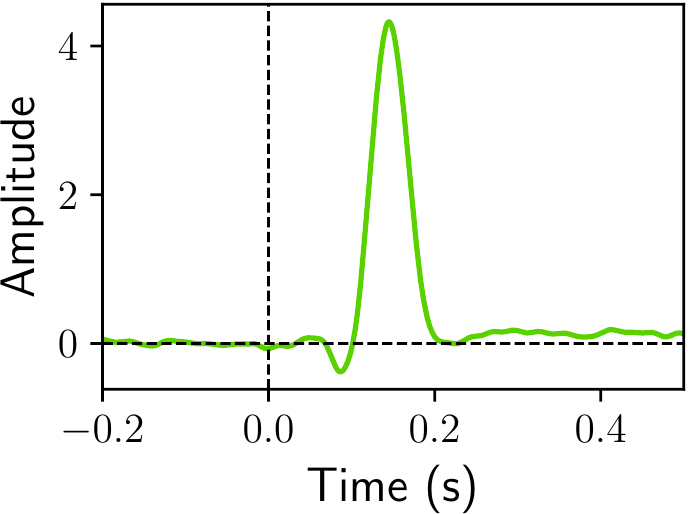}%
\raisebox{0.2\height}{\includegraphics[width=0.68\textwidth]{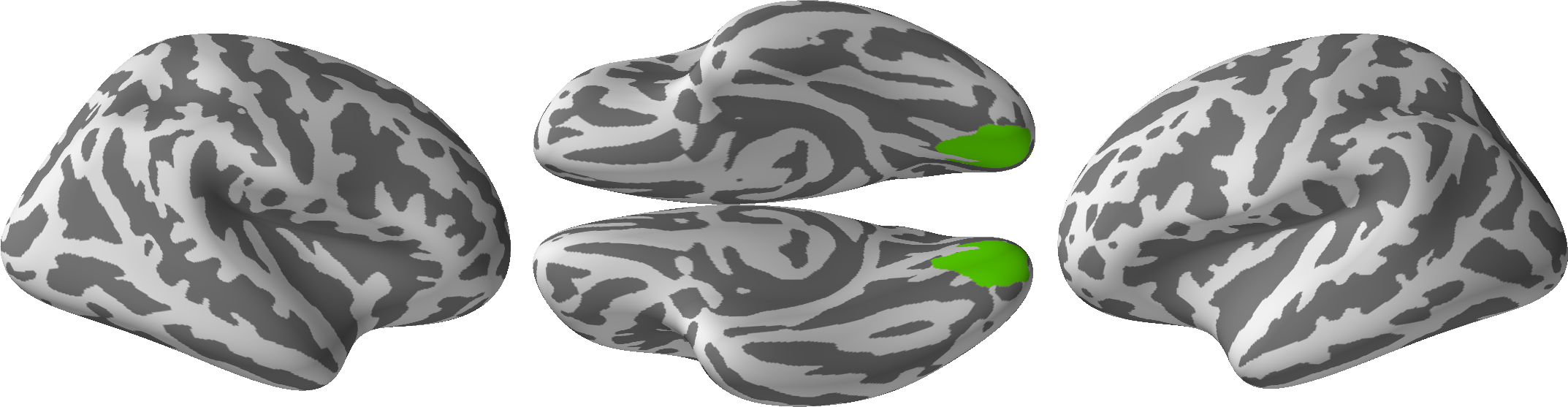}} \\
\includegraphics[width=0.3\textwidth]{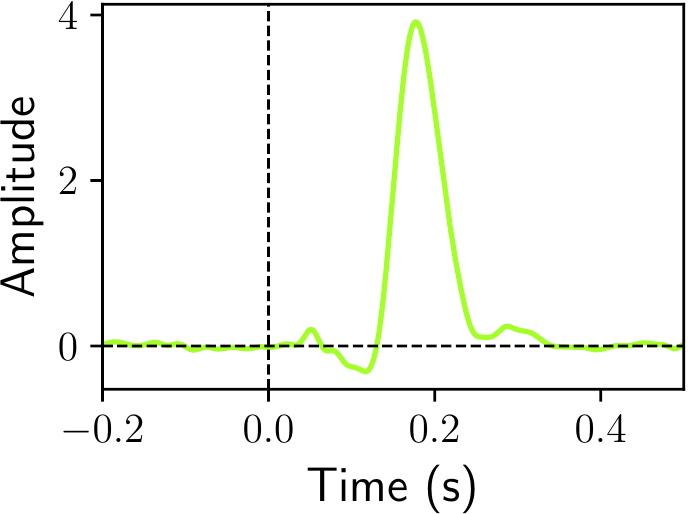}%
\raisebox{0.2\height}{\includegraphics[width=0.68\textwidth]{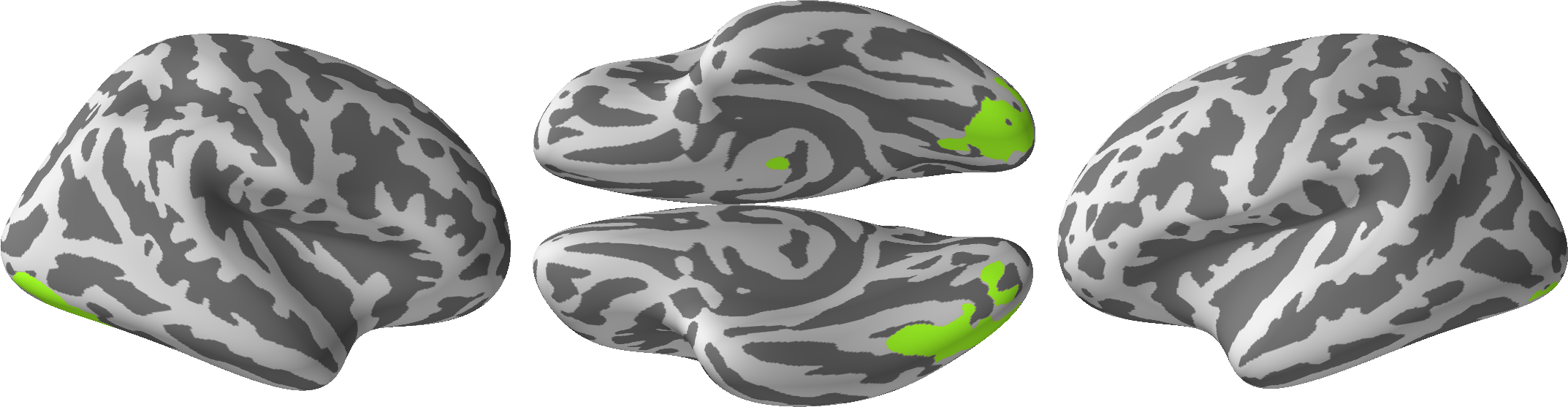}} \\
\includegraphics[width=0.3\textwidth]{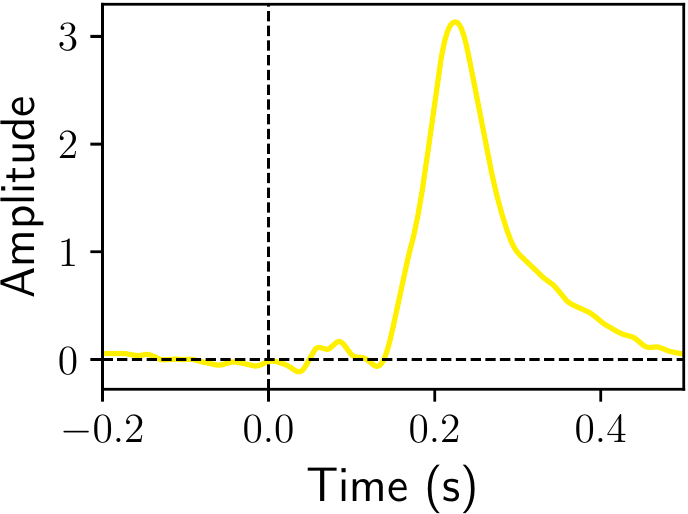}%
\raisebox{0.2\height}{\includegraphics[width=0.68\textwidth]{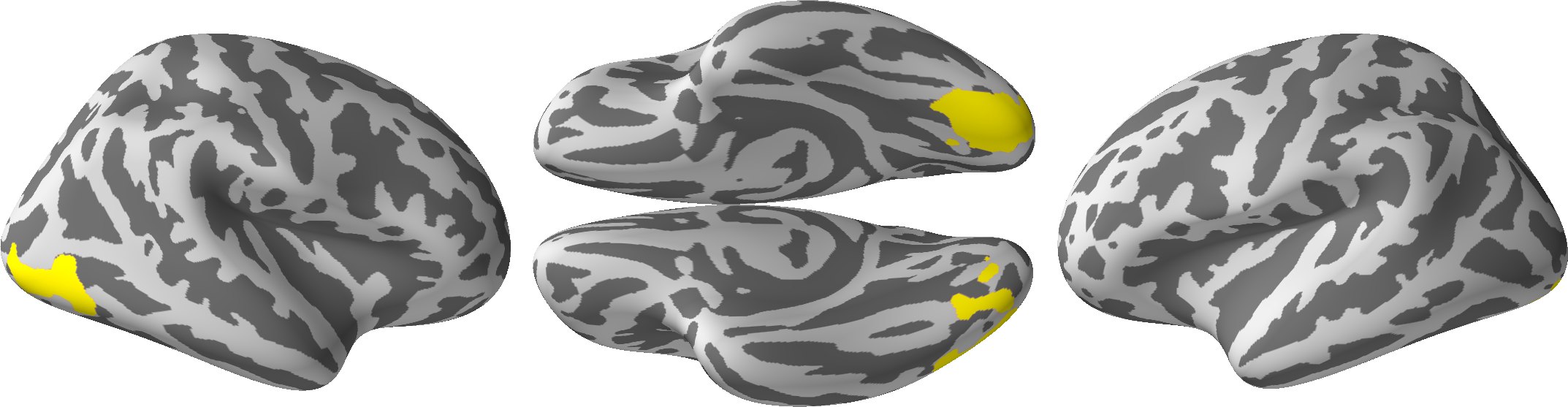}} \\
\includegraphics[width=0.3\textwidth]{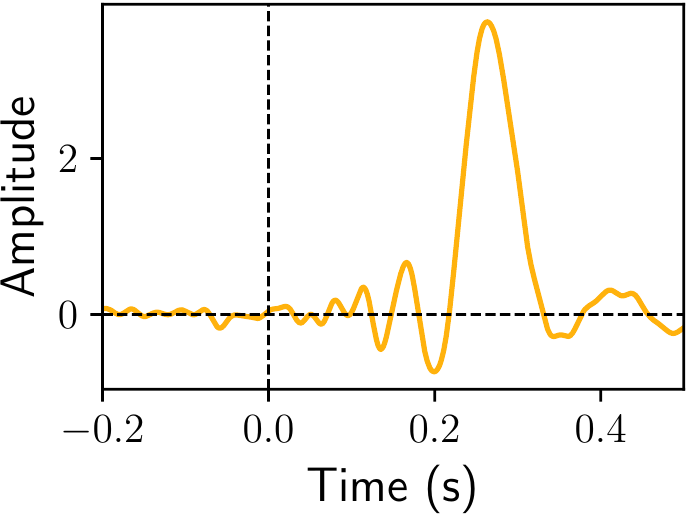}%
\raisebox{0.2\height}{\includegraphics[width=0.68\textwidth]{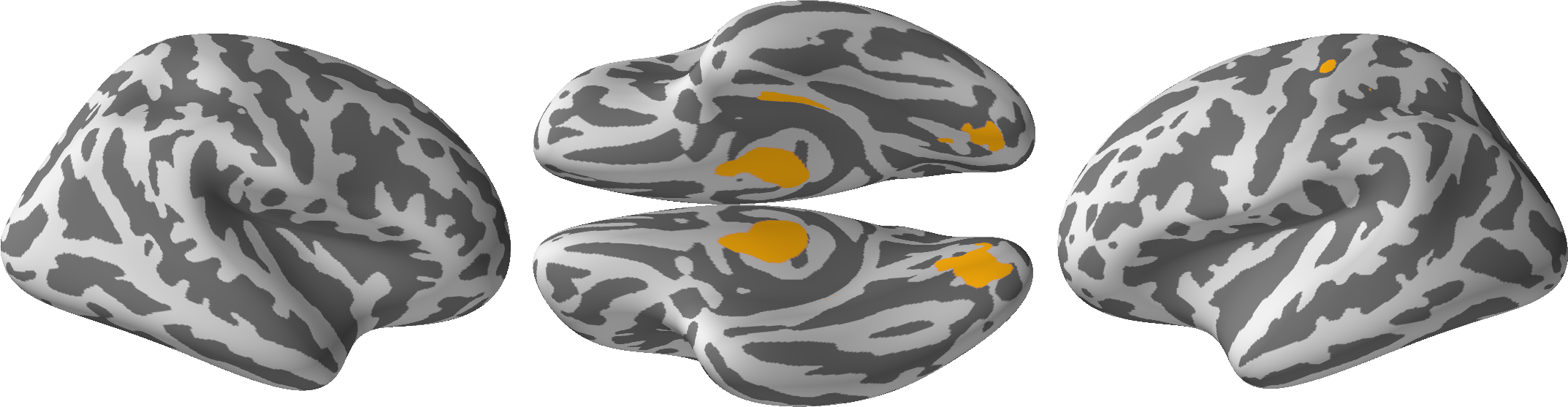}} \\
\includegraphics[width=0.3\textwidth]{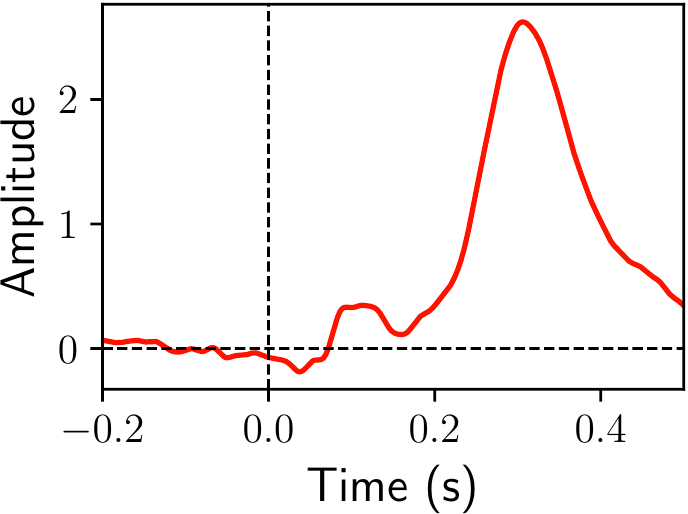}%
\raisebox{0.2\height}{\includegraphics[width=0.68\textwidth]{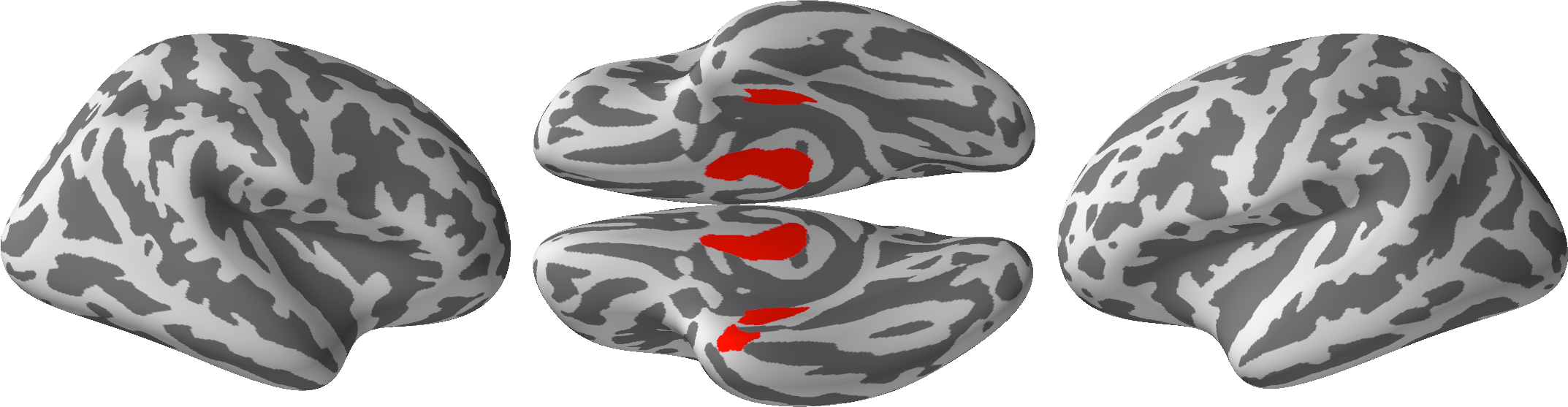}} \\
\includegraphics[width=0.3\textwidth]{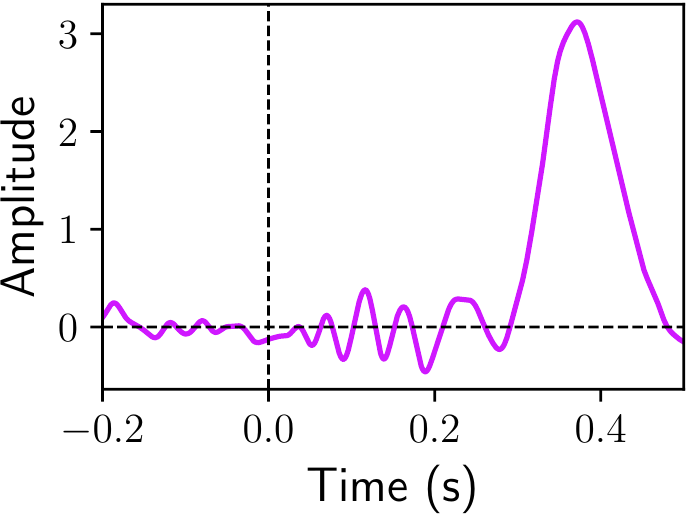}%
\raisebox{0.2\height}{\includegraphics[width=0.68\textwidth]{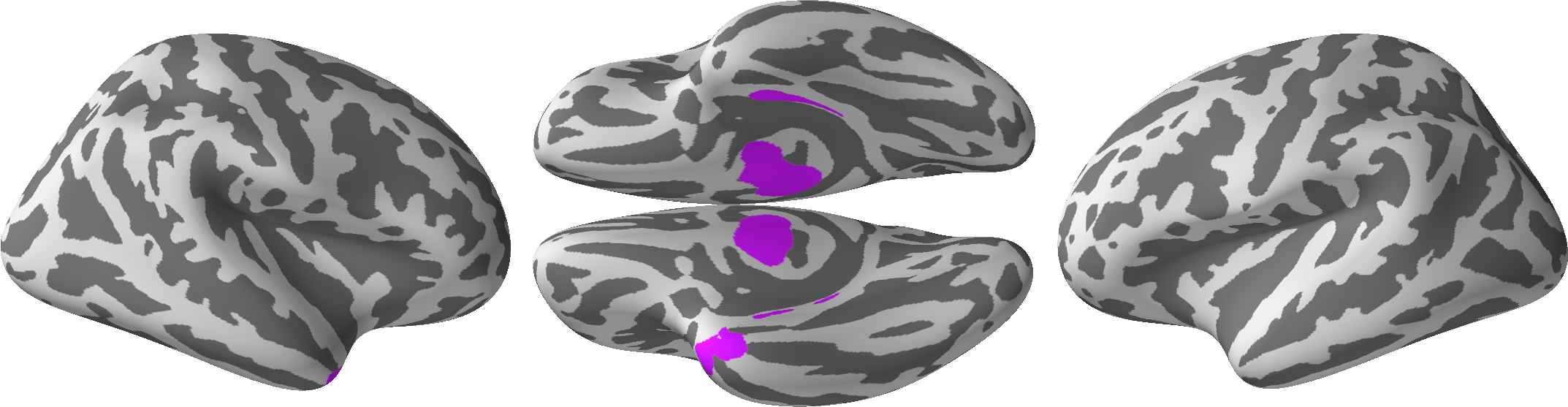}} \\
\includegraphics[width=0.3\textwidth]{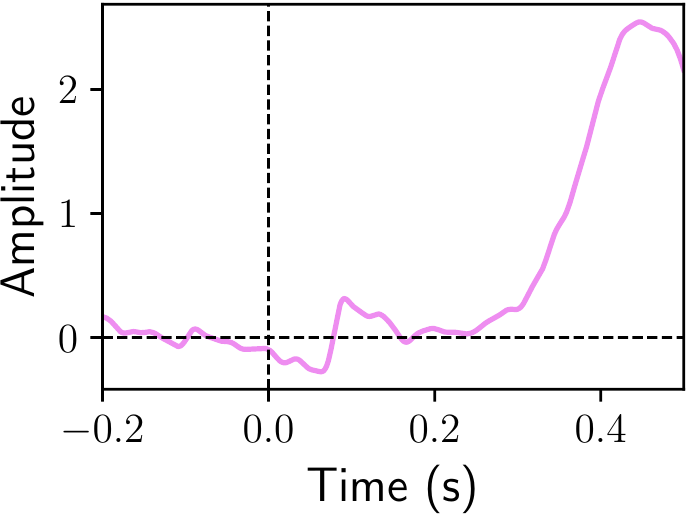}%
\raisebox{0.2\height}{\includegraphics[width=0.68\textwidth]{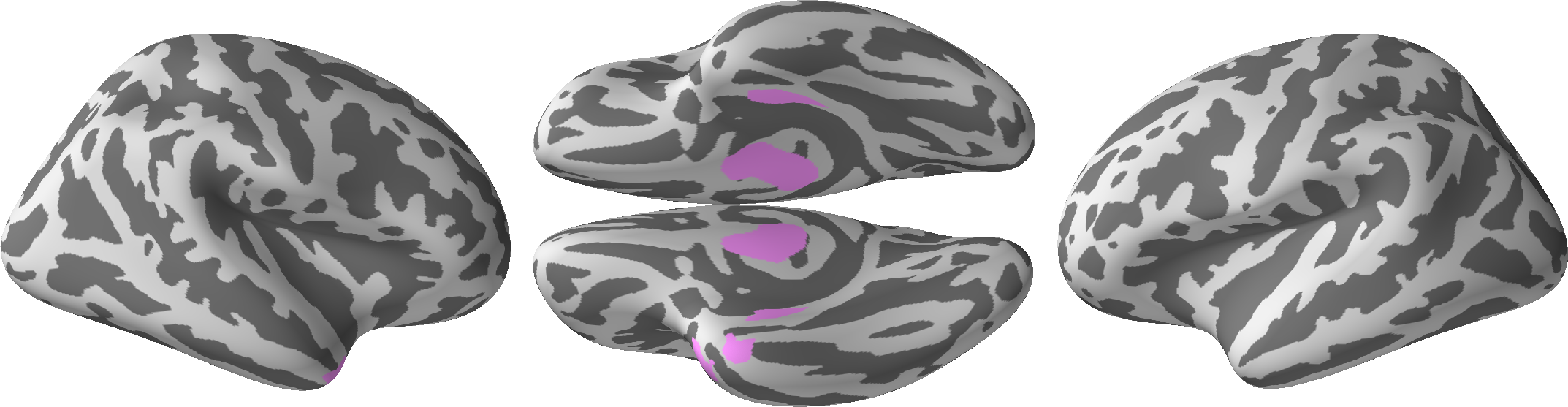}} \\
}

\section{Average forward operators on fMRI datasets}
\label{sec:spatial_maps}
We display the average forward operator across subjects on the Raiders, Forrest, Clips and Sherlock datasets obtained with MultiViewICA and GroupICA with 5 components. A 5~mm spatial smoothing was applied on all datasets, and the confound signals corresponding to the 5 components with the highest variance were removed before applying MultiViewICA or GroupICA.

{\centering

  \includegraphics[width=\textwidth]{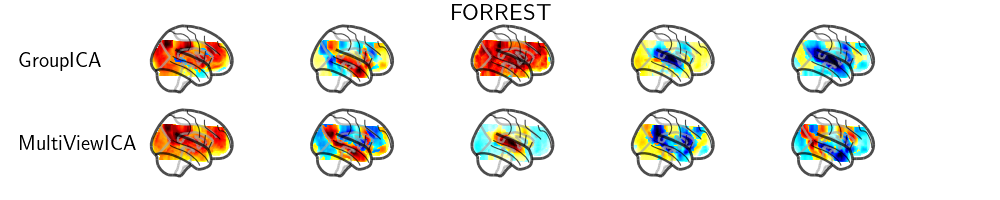} \\
  \includegraphics[width=\textwidth]{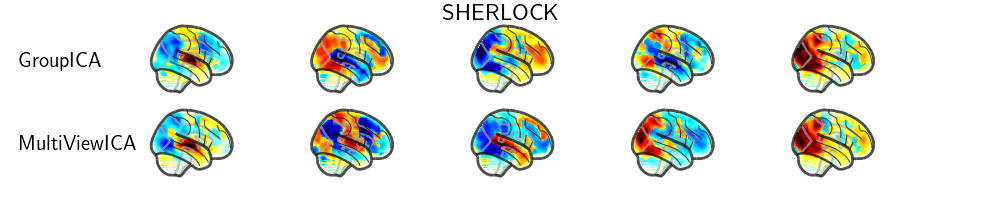} \\
  \includegraphics[width=\textwidth]{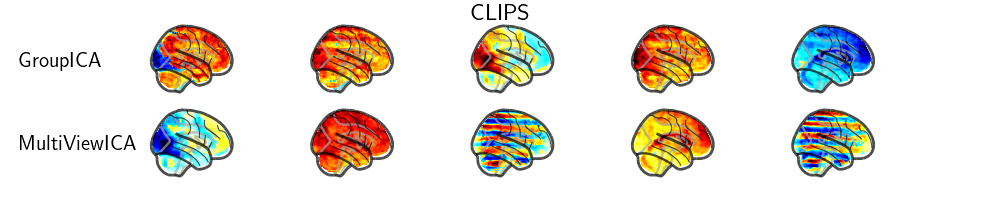} \\
  \includegraphics[width=\textwidth]{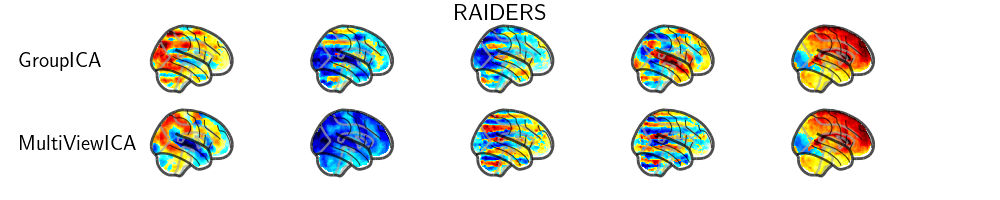} \\
}

\section{Synthetic benchmark using the model $\xb^i = A^i\sbb + \nb^i$}
\label{app:complex_cov}
We generate data according to the model $\xb^i = A^i\sbb + \nb^i$, where $\xb^i \in \mathbb{R}^{50}$, $\sbb \in \mathbb{R}^{20}$, and $\nb^i\sim \mathcal{N}(0, \sigma^2 I_{50})$. After applying individual PCA to obtain signals of dimension $20$, we apply the different ICA algorithms and report the reconstruction error in fig.~\ref{fig:reconstruction_synth}.

\begin{figure}[H]
  \center
  \includegraphics[width=0.5\linewidth]{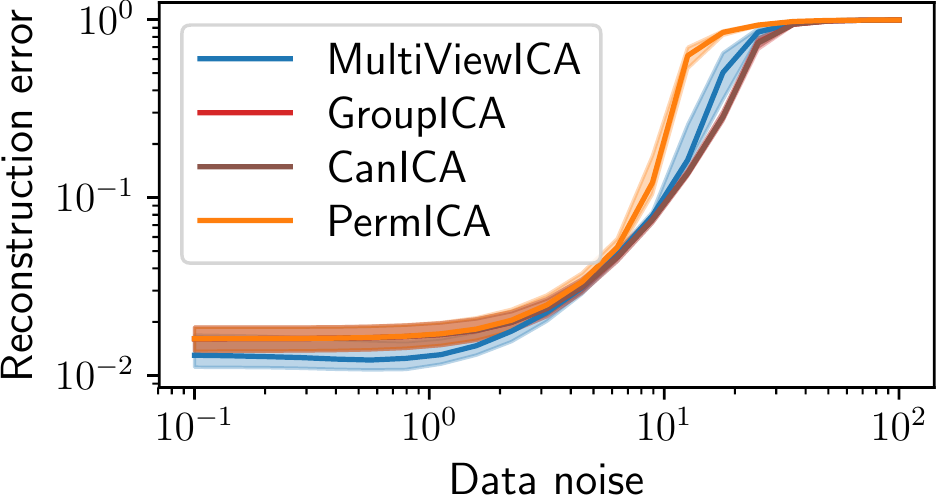}
  \caption{Synthetic experiment with model $\xb^i = A^i\sbb^i + \nb^i$} 
  \label{fig:reconstruction_synth}
\end{figure}

\section{Summary of our quantitative results}
\label{sec:app_real_data}
Our quantitative results for the fMRI experiments of time-segment matching and BOLD signal reconstruction and on for the MEG phantom data experiment are summarized, respectively, in Table~\ref{tab:timeseg}, Table~\ref{tab:recon} and Table~\ref{tab:meg}. All methods are compared upon extraction of sources with the same dimensionality ($20$ components).

\begin{table}
    \centering
    \begin{tabular}{|c|c | c | c|}
            \hline
         \textbf{Dataset} & \textbf{Method} & \textbf{Accuracy} & \textbf{Confidence interval} \\
         \hline
clips   & Chance      & 0.002&[0.001, 0.003] \\
        & CanICA    & 0.130&[0.112, 0.147] \\
        & PCA + GroupICA      & 0.124&[0.109, 0.139] \\
        & GroupICA    & 0.152&[0.133, 0.171] \\

        & PermICA     & 0.147&[0.126, 0.169] \\
        & SRM         & 0.115&[0.104, 0.126] \\
        & MultiViewICA& \textbf{0.167}&[0.142, 0.192] \\
        \hline
forrest & Chance      & 0.002&[0.001, 0.002] \\
        & CanICA    & 0.192&[0.170, 0.214] \\
        & PCA + GroupICA      & 0.088&[0.077, 0.098] \\
        & GroupICA    & 0.154&[0.137, 0.170] \\
        & PermICA     & 0.135&[0.118, 0.152] \\
        & SRM         & 0.188&[0.173, 0.203] \\
        & MultiViewICA& \textbf{0.448}&[0.411, 0.484] \\
        \hline
raiders & Chance      & 0.002&[0.001, 0.003] \\
        & CanICA    & 0.256&[0.220, 0.291] \\
        & PCA + GroupICA      & 0.331&[0.289, 0.372] \\
        & GroupICA    & 0.321&[0.281, 0.361] \\
        & PermICA     & 0.381&[0.341, 0.421] \\
        & SRM         & 0.265&[0.240, 0.289] \\
         & MultiViewICA& \textbf{0.408}&[0.358, 0.458] \\
         \hline
sherlock& Chance      & 0.005&[0.003, 0.006] \\
        & CanICA    & 0.607&[0.567, 0.648] \\
        & PCA + GroupICA      & 0.454&[0.416, 0.492] \\
        & GroupICA    & 0.519&[0.481, 0.556] \\
        & PermICA     & 0.399&[0.365, 0.434] \\
        & SRM         & 0.493&[0.465, 0.520] \\
        & MultiViewICA& \textbf{0.873}&[0.844, 0.903] \\
\hline
    \end{tabular}
    \caption{Timesegment matching: Summary of our quantitative results. We report the mean accuracy across cross-validation splits.}
    \label{tab:timeseg}
\end{table}

\begin{table}
    \centering
    \begin{tabular}{|c|c | c | c|}
            \hline
         \textbf{Dataset} & \textbf{Method} & \textbf{R2 score} & \textbf{Confidence interval} \\
         \hline
         clips   & Chance              & 0.000&[0.000 ,0.000] \\
        & CanICA            &  0.110&[ 0.097 , 0.123] \\
        & PCA + GroupICA              &  0.075&[ 0.058 , 0.092] \\
        & GroupICA            &  0.077&[ 0.059 , 0.094] \\
        & PermICA             &  0.099&[ 0.087 , 0.111] \\
        & SRM                 &  0.081&[ 0.069 , 0.094] \\
        & MultiViewICA        &  \textbf{0.114}&[ 0.099 , 0.128] \\
        \hline
forrest & Chance              & 0.000&[0.000 ,0.000] \\
        & CanICA            &  0.181&[ 0.169 , 0.193] \\
        & PCA + GroupICA              &  0.072&[ 0.054 , 0.090] \\
        & GroupICA            &  0.081&[ 0.062 , 0.099] \\
        & PermICA             &  0.098&[ 0.090 , 0.106] \\
        & SRM                 &  0.180&[ 0.168 , 0.193] \\
        & MultiViewICA        &  \textbf{0.191}&[ 0.177 , 0.204] \\
        \hline
raiders & Chance              & 0.000&[0.000 ,0.000] \\
        & CanICA            &  0.136&[ 0.122 , 0.149] \\
        & PCA + GroupICA              &  0.063&[ 0.045 , 0.080] \\
        & GroupICA            &  0.062&[ 0.043 , 0.081] \\
        & PermICA             &  0.107&[ 0.091 , 0.124] \\
        & SRM                 &  0.138&[ 0.121 , 0.154] \\
        & MultiViewICA        &  \textbf{0.144}&[ 0.124 , 0.164] \\
        \hline
sherlock& Chance              & 0.000&[0.000 ,0.000] \\
        & CanICA            &  0.156&[ 0.141 , 0.172] \\
        & PCA + GroupICA              &  0.087&[ 0.065 , 0.108] \\
        & GroupICA            &  0.091&[ 0.070 , 0.112] \\
        & PermICA             &  0.067&[ 0.055 , 0.078] \\
        & SRM                 &  \textbf{0.164}&[ 0.147 , 0.181] \\
        & MultiViewICA        &  0.161&[ 0.142 , 0.180] \\
        \hline

    \end{tabular}
    \caption{Reconstructing the BOLD signal of missing subjects: Summary of our quantitative results. We report the mean R2 score across cross-validation splits.}
    \label{tab:recon}
\end{table}

\begin{table}
    \centering
    \begin{tabular}{|c|c|c|c}
    \hline
         \textbf{Method} & \textbf{Reconstruction error} & \textbf{1st and 3d quartiles} 
         \\
         \hline
         MultiViewICA & \textbf{0.0045} & [0.0039, 0.0052] \\ 
GroupICA & 0.1098 & [0.0549, 0.1734] \\ 
PCA+GroupICA & 0.1111 & [0.0760, 0.1502] \\ 
PermICA & 0.0730 & [0.0423, 0.1037] \\ 
\hline
    \end{tabular}
    \caption{Phantom MEG data: Summary of our quantitative results with 2 epochs. We report the median reconstruction error across cross-validation splits.}
    \label{tab:meg}
\end{table}

\end{document}